\documentclass[a4paper,11pt]{article}
\pdfoutput=1
\usepackage{a4wide}
\usepackage{amsmath,amsfonts,amssymb,amsthm}
\usepackage[colorlinks,citecolor=blue]{hyperref}
\usepackage{enumitem}
\usepackage{bm}
\usepackage{natbib}
\usepackage{microtype}
\usepackage{graphicx}

\newtheorem{theorem}{Theorem}[section]
\newtheorem{lemma}[theorem]{Lemma}
\newtheorem{proposition}[theorem]{Proposition}
\newtheorem{corollary}[theorem]{Corollary}

\theoremstyle{definition}
\newtheorem{definition}[theorem]{Definition}

\newtheorem{remark}[theorem]{Remark}

\numberwithin{equation}{section}
\numberwithin{figure}{section}
\numberwithin{table}{section}

\def \bN {\mathbb{N}}

\def \bR {\mathbb{R}}

\def \bZ {\mathbb{Z}}

\def \cB {\mathcal{B}}

\def \cE {\mathcal{E}}
\def \cF {\mathcal{F}}
\def \cG {\mathcal{G}}
\def \cH {\mathcal{H}}

\def \cN {\mathcal{N}}
\def \cO {\mathcal{O}}

\def \cQ {\mathcal{Q}}

\def \cS {\mathcal{S}}

\def \Ba {{\boldsymbol{a}}}
\def \Bb {{\boldsymbol{b}}}

\def \Be {{\boldsymbol{e}}}

\def \Bk {{\boldsymbol{k}}}

\def \Bm {{\boldsymbol{m}}}
\def \Bn {{\boldsymbol{n}}}

\def \Bx {{\boldsymbol{x}}}
\def \By {{\boldsymbol{y}}}

\def \sgn {\,{\rm sgn}\,}

\def \Pdim {\,{\rm Pdim}\,}
\def \Bin {\,{\rm Bin}\,}
\def \mid {\,{\rm mid}\,}

\begin{document}

\title{Approximation in shift-invariant spaces with deep ReLU neural networks}

\author{
Yunfei Yang \thanks{Department of Mathematics, The Hong Kong University of Science and Technology, Clear Water Bay, Kowloon, Hong Kong, China. (Corresponding author, Email: yyangdc@connect.ust.hk)}
\and
Zhen Li \thanks{Theory Lab, Huawei Technologies Co., Ltd., Shenzhen, China.}
\and
Yang Wang \thanks{Department of Mathematics, The Hong Kong University of Science and Technology, Clear Water Bay, Kowloon, Hong Kong, China.}
}
\date{}
\maketitle

\begin{abstract}
We study the expressive power of deep ReLU neural networks for approximating functions in dilated shift-invariant spaces, which are widely used in signal processing, image processing, communications and so on. Approximation error bounds are estimated with respect to the width and depth of neural networks. The network construction is based on the bit extraction and data-fitting capacity of deep neural networks.  As applications of our main results, the approximation rates of classical function spaces such as Sobolev spaces and Besov spaces are obtained. We also give lower bounds of the $L^p (1\le p \le \infty)$ approximation error for Sobolev spaces, which show that our construction of neural network is asymptotically optimal up to a logarithmic factor. 

\vspace{0.3cm}
\end{abstract}

\section{Introduction}

In the past few years, machine learning techniques based on deep neural networks have been remarkably successful in many applications such as computer vision, natural language processing, speech recognition and even art creating \citep{lecun2015deep, gatys2016image}. Despite their state-of-the-art performance in practice, the fundamental theory behind deep learning remains largely unsolved, including function representation, optimization, generalization and so on. One cornerstone in the theory of neural networks is their expressive power, which has been studied by many pioneer researchers in many different aspects such as VC-dimension and Pseudo-dimension \citep{bartlett1999almost,goldberg1995bounding,bartlett2019nearly}, number of linear sub-domains \citep{montufar2014number,raghu2017expressive,serra2018bounding}, data-fitting capacity \citep{yun2019small,vershynin2020memory} and data compression \citep{bolcskei2019optimal,elbrachter2021deep}.

In this paper, we study the expressive power of deep ReLU neural networks in terms of their capability of  approximating functions. It is well known that, under certain mild conditions on the activation function, two-layer neural networks are universal. They can approximate continuous functions arbitrarily well on compact set, if the width of network is allowed to grow arbitrarily large \citep{cybenko1989approximation,hornik1991approximation,pinkus1999approximation}. Recently, the universality of neural networks with fixed width have also been established in \citet{hanin2019universal,hanin2017approximating}. A further question is about the order of approximation error, or equivalently, the required size of a neural network that is sufficient for approximating a given class of functions, determined by the application at hand, to a prescribed accuracy. The study of this question mainly focused on shallow neural networks in the 1990s. Recent breakthrough of deep learning in practical areas has attracted many researchers to work on estimating approximation error of deep neural networks on different types of function classes, such as continuous functions \citep{yarotsky2018optimal}, band-limited functions \citep{montanelli2019deep}, smooth functions \citep{lu2020deep} and piecewise smooth functions \citep{petersen2018optimal}.

The purpose of this paper is to approximate functions in dilated shift-invariant spaces using neural networks. 
More specifically, we construct deep ReLU neural networks to approximate functions of the form 
\[
g(\Bx) = \sum_{\Bn\in\bZ^d} c_\Bn \varphi(2^j\Bx-\Bn),
\]
which are functions in the dilated shift-invariant spaces generated by a continuous function $\varphi$.
Our main contribution is that we provide a systematical way to construct such neural networks and that we characterize their expressive power by rigorous estimation of the approximation error. Our work is closely related to signal processing, image processing, communication of information and so on, for in these areas, shift-invariant spaces are widely used  \citep{grochenig2001foundations,mallat1999wavelet}. For example, digital signals transmitted in communication systems are expressed by functions in these spaces \citep{Oppenheim_2009}.
Recently, many efforts are made to apply neural networks to solve problems in these areas \citep{purwins2019deep,yu2010deep,ker2017deep,mousavi2015deep,kiranyaz20191,fan2020advancing}. 
Despite their success in practice, theoretical understanding of deep learning in such applications still remains open. We hope that our work provides a theoretical justification and explanation for the application of deep neural networks in such areas.

Our results on shift-invariant spaces can also be used to study the approximation on other function spaces. Shift-invariant spaces are closely related to wavelets \citep{daubechies1992ten,mallat1999wavelet}, which can be used to approximate classical function spaces such as Sobolev spaces, Besov spaces and so on \citep{de1994approximation,jia1993approximation,lei1997approximation,kyriazis1995approximation,jia2004approximation,jia2010approximation}. By combining our construction with these existing results, we can estimate approximation errors of Sobolev functions and Besov functions by deep neural networks,
which generalize the results of Yarotsky \citep{yarotsky2017error,yarotsky2018optimal,yarotsky2019phase} and Shen et al. \citep{shen2019nonlinear,shen2019deep,lu2020deep}. Besides, we also give lower bounds of the approximation error using the nonlinear $n$-width introduced by \citet{ratsaby1997value,maiorov1999degree}. It is worth to point out that our lower bounds hold for $L^p$ error with $1\le p\le \infty$, while, as far as we know, it is only proved for $L^\infty$ error in the literature. These lower bounds indicate the asymptotic optimality of our error estimates on Sobolev spaces.

The rest of this paper is organized as follows. Notations and necessary terminology are summarized in section \ref{sec:Preliminaries}. A detailed discussion of our main results is presented in section \ref{sec:shift_invariant_spaces}. In section \ref{sec:Sobolev_Besov}, we apply our main theorem to Sobolev spaces and Besov spaces, and show the optimality of the approximation result in Sobolev spaces. In section \ref{sec:discussion}, we make a summary of our result and discuss the its relation with other studies. Finally, the detail of the network construction and the proofs of main theorems are contained in sections \ref{sec:proof_Q} and \ref{sec:proof_uniform}.

\section{Preliminaries}\label{sec:Preliminaries}

\subsection{Notations}\label{sec:Notations}

Let us first introduce some notations.  We denote the set of positive integers by $\bN=\{1,2,\dots\}$. For each $j\in\bN$, we denote $\bZ^d_j := [0,2^j-1]^d\cap \bZ^d$. Hence, the cardinality of $\bZ^d_j$ is $|\bZ^d_j| = 2^{jd}$. Assume $\Bn \in \bN^d$, the asymptotic notation $f(\Bn) =\cO(g(\Bn))$ means that there exists $M,C>0$ independent of $\Bn$ such that $f(\Bn)\le Cg(\Bn)$ for all $\|\Bn\|_{\ell^\infty} \ge M$. The notation $f(\Bn) \asymp g(\Bn)$ means that $f(\Bn) =\cO(g(\Bn))$ and $g(\Bn) =\cO(f(\Bn))$. For any $x\in [0,1)$, we denote the binary representation of $x$ by 
\[
\Bin 0.x_1x_2\cdots=\sum_{i=1}^\infty 2^{-i}x_i=x,
\]
where each $x_i\in \{0,1\}$ and $\liminf_{i\to \infty} x_i \neq 1$. Notice that the binary representation defined in this way is unique for $x\in [0,1)$.

We will need the following notation to approximately partition $[0,1]^d$ into small cubes. For any $j,d\in \bN$, let $0<\delta<2^{-j}$, we denote 
\begin{equation}\label{Q(j,delta,1)}
Q(j,\delta,1):= [0,1)\setminus \cup_{k=1}^{2^j-1}(k2^{-j}-\delta,k2^{-j}),
\end{equation}
and for $d\ge 2$,
\begin{equation}\label{Q(j,delta,d)}
Q(j,\delta,d) := \{ \Bx=(x_1,\dots,x_d): x_i\in Q(j,\delta,1), 1\le i\le d \}.
\end{equation}
Figure \ref{example of Q} shows an example of $Q(j,\delta,d)$.

\begin{figure}[ht]\label{example of Q}
\centering
\includegraphics[scale=.37]{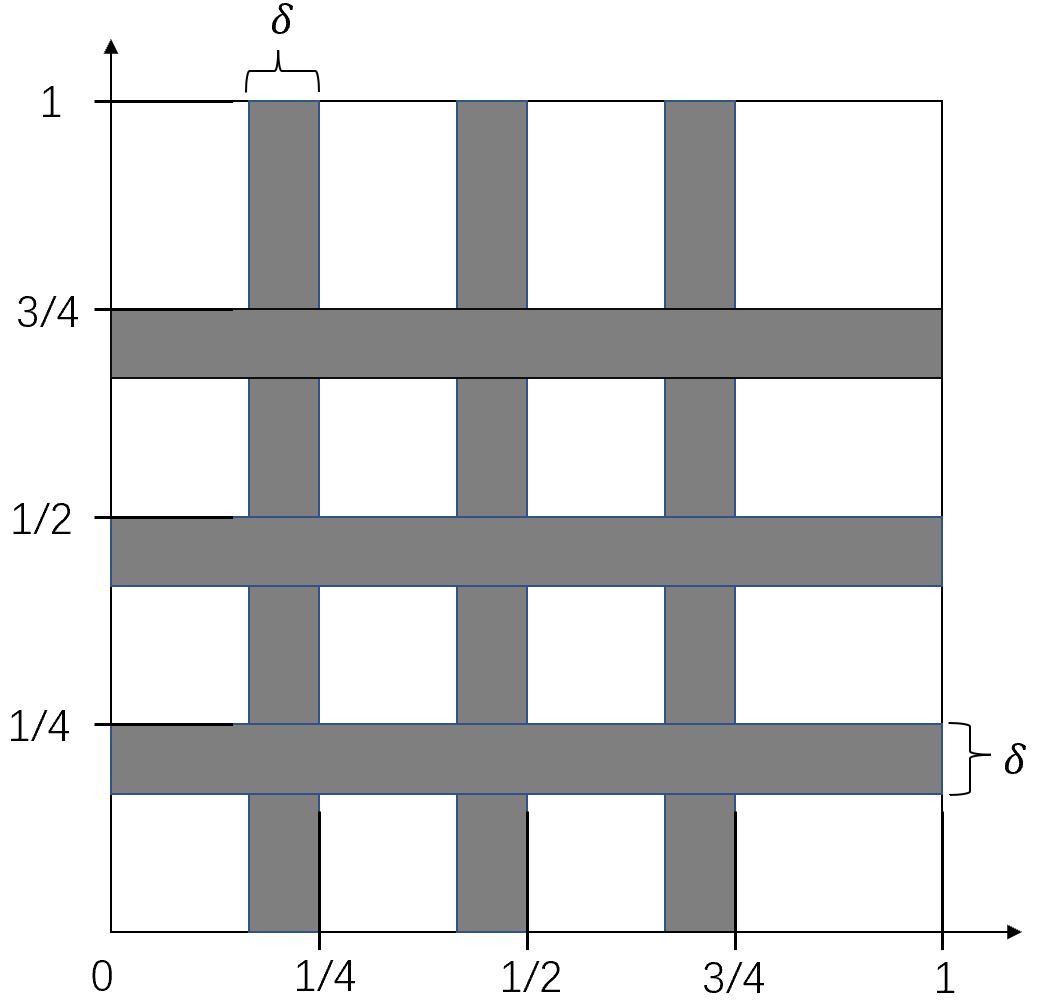}
\caption{An example of $Q(j,\delta,d)$ with $j=2$ and $d=2$. It is the union of the white region in $[0,1]^d$.}
\label{figure_Q}
\end{figure}

Finally, for any function $f: \Omega \subseteq \bR \to \bR$, we will extend its definition to $\Omega^d$ by applying $f$ coordinate-wisely to $\Bx=(x_1,\dots,x_d)\in \Omega^d$, i.e. $f(\Bx) := (f(x_1),\dots,f(x_d))$, without further notification.

In Table \ref{symbols}, we summarize a set of symbols that are used throughout this paper. Some of the notations will be introduced later.

\begin{table}[h]
\centering
\begin{tabular}{|c|c|}
\hline
Notation & Definition \\
\hline
$\bZ^d_j$ & $[0,2^j-1]^d\cap \bZ^d$ \\
\hline
$\Bin 0.x_1x_2\cdots$ & Binary representation of $x= \sum_{i=1}^\infty 2^{-i}x_i$, $x_i\in \{0,1\}$ \\
\hline
$Q(j,\delta,d)$ & Approximately partition of $[0,1]^d$, Eq.(\ref{Q(j,delta,d)}) \\
\hline
$\cN\cN(N,L)$ & Function class of neural networks with width $N$ and depth $L$ \\
\hline
$\cS_j(\varphi,M)$ & Dilated shift-invariant space generated by $\varphi$, Eq.(\ref{S_j(varphi,M)}) \\
\hline
$\cE(f,\cH;\cB)$ & Approximation error of $f$ from $\cH$ in the norm of $\cB$, Eq.(\ref{E(f,H,B)}) \\
\hline
$m_j:[0,1)\to\bZ_j$ & $m_j(x)=\lfloor 2^j x\rfloor$, Lemma \ref{float_representation} \\
\hline
$r_j:[0,1)\to [0,1)$ & $r_j(x)=2^j x-m_j(x)$, Lemma \ref{float_representation} \\
\hline
$\bZ^d_\varphi$ & $\{\Bn\in \bZ^d: \exists \Bx\in [0,1)^d\ s.t.\ \varphi(\Bx-\Bn) \neq 0  \}$, Lemma \ref{float_representation} \\
\hline
$C_\varphi$ & $|\bZ^d_\varphi|$, the cardinality of $\bZ^d_\varphi$ \\
\hline
\end{tabular}
\caption{\label{symbols}A list of notations used in this paper.}
\end{table}

\subsection{Neural networks}

In this paper, we only consider feed-forward neural networks with ReLU activation function $\sigma(x) :=\max\{0,x\}$. Let $2\le L\in \bN$ and $N_0,\dots,N_L\in \bN$. We say $\eta = (A^{(\ell)},\Ba^{(\ell)})_{\ell=1}^L$ is a network architecture, if $A^{(\ell)}\in \bR^{N_{\ell}\times N_\ell-1}$, $\Ba^{(\ell)}\in \bR^{N_\ell}$ and each entry of $A^{(\ell)}$ and $\Ba^{(\ell)}$ is in $\{0,1\}$. We say a function $f:\bR^{N_0}\to \bR^{N_L}$ can be implemented (or represented) by a neural network with architecture $\eta$ if it can be written in the form
\[
f(\Bx) = T_L(\sigma(T_{L-1}(\cdots \sigma(T_1(\Bx))\cdots))),
\]
where $T_\ell(\Bx) := (A^{(\ell)} \odot B^{(\ell)}) \Bx +\Ba^{(\ell)} \odot \Bb^{(\ell)}$ is an affine transformation with $B^{(\ell)} \in \bR^{N_\ell\times N_{\ell-1}}$ and $\Bb^{(\ell)}\in \bR^{N_\ell}$, and $\odot$ is entry-wise product. $L$ is called the depth of neural network. The width is referred to $N=\max\{N_1,\dots,N_{L-1} \}$. The number of parameters of the architecture $\eta$ is $W= \sum_{\ell=1}^L \|A^{(\ell)}\|_{\ell^0} + \|\Ba^{(\ell)}\|_{\ell^0}$ and the number of (hidden) neurons is $U=\sum_{\ell=1}^{L-1} N_\ell$.

We will mainly focus on fully connected neural networks, which we refer to the case that all entries of $A^{(\ell)}$ and $\Ba^{(\ell)}$ are ones. Hence, we have no restriction on the coefficients of the affine maps $T_\ell(\Bx) := B^{(\ell)} \Bx + \Bb^{(\ell)}$. When the input dimension $N_0$ and output dimension $N_L$ are clear from contexts, we denote by $\cN\cN(N,L)$ the set of functions that can be represented by neural networks with width at most $N$ and depth at most $L$.
The expression ``a neural network $\phi$ with width $N$ and depth $L$'' means $\phi \in \cN\cN(N,L)$.

\subsection{Shift-invariant spaces}

Let $\varphi:\bR^d\to \bR$ be a continuous function with compact support. The shift-invariant space $\cS(\varphi)$ generated by $\varphi$ is the set of all finite linear combinations of the shifts of $\varphi$, i.e. liner combination of $\varphi(\cdot -\Bn)$ with $\Bn \in \bZ^d$. For each $j\ge 0$, the dilated shift-invariant space $\cS_j(\varphi)$ is defined to be the dilation of $\cS(\varphi)$ by $2^j$. That is, every function $g\in S_j(\varphi)$ is of the form
\begin{equation}\label{summation}
g(\Bx) = \sum_{\Bn\in\bZ^d} c_\Bn \varphi(2^j\Bx-\Bn),
\end{equation}
where $(c_\Bn)_{\Bn\in\bZ^d}$ is zero except for finitely many $\Bn$. Note that the space $\cS_j(\varphi)$ is invariant under the translations $T_{2^{-j} \Bm} g(\Bx) := g(\Bx-2^{-j} \Bm)$ with $\Bm\in \bZ^d$. For any $M>0$, we denote
\begin{equation}\label{S_j(varphi,M)}
\cS_j(\varphi,M):= \left\{\sum_{\Bn\in\bZ^d} c_\Bn \varphi(2^j\Bx-\Bn)\in S_j(\varphi): |c_\Bn|<M \text{ for any } \Bn\in \bZ^d \right\}.
\end{equation}

\subsection{Sobolev spaces and Besov spaces}

For $1\le p\le \infty$, the $p$-norm of $L^p(\bR^d)$ is denoted by $\|\cdot\|_p$ for convenience. Let $k\in \bN$, the Sobolev space $W^{k,p}(\bR^d)$ is the set of functions $f\in L^p(\bR^d)$ which have finite Sobolev norm
\[
\|f\|_{W^{k,p}} := \left( \sum_{\|\boldsymbol{\alpha}\|_{\ell^1}\le k} \|D^{\boldsymbol{\alpha}} f\|^p_p \right)^{1/p},
\]
where $D^{\boldsymbol{\alpha}}$ is the weak derivative of order $\boldsymbol{\alpha}$.
There are several ways to generalize the definition of Sobolev norms to non-integer regularity. Here, we introduce the Besov spaces. Let us denote the difference operator by $\Delta_\By f(\Bx):= f(\Bx+\By)-f(\Bx)$ for any $\Bx, \By\in\bR^d$. Then, for any positive integer $m$, the $m$-th modulus of smoothness of a function $f\in L^p(\bR^d)$ is defined by
\[
\omega_m(f,h)_p := \sup_{\|\By\|_{\ell^2}\le h} \|\Delta_\By^m (f)\|_p, \quad h\ge 0,
\]
where
\[
\Delta_\By^m (f)(\Bx) := \sum_{j=0}^{m} \binom{m}{j} (-1)^{m-j}f(\Bx+j\By).
\]
For $\mu >0$ and $1\le p,q\le \infty$, the Besov space $B^\mu_{p,q}(\bR^d)$ is the collection of functions $f\in L^p(\bR^d)$ that have finite semi-norm $|f|_{B^\mu_{p,q}}<\infty$, where the semi-norm is defined as
\[
|f|_{B^\mu_{p,q}}:= \left\{
\begin{aligned}
&\left( \int_{0}^{\infty} \left| \frac{\omega_m(f,t)_p}{t^\mu} \right|^q \frac{dt}{t}  \right)^{1/q}, \quad &&1\le q<\infty, \\
&\sup_{t>0} \frac{\omega_m(f,t)_p}{t^{\mu}}, \quad &&q=\infty,
\end{aligned}
\right.
\]
where $m$ is an integer larger than $\mu$. The norm for  $B^\mu_{p,q}$ is
\[
\| f\|_{ B^\mu_{p,q}} := \|f\|_p + |f|_{B^\mu_{p,q}}.
\]
Note that for $k\in \bN$, we have the embedding $B^k_{p,1} \hookrightarrow W^{k,p} \hookrightarrow B^k_{p,\infty}$ and $B^k_{2,2}=W^{k,2}$. A general discussion of Sobolev spaces and Besov spaces can be found in \citet{devore1993constructive}.

\subsection{Approximation}

Let $\cB$ be a normed space and $f\in \cB$, we denote the approximation error of $f$ from a set $\cH\subseteq \cB$ under the norm of $\cB$ by
\begin{equation}\label{E(f,H,B)}
\cE(f,\cH;\cB) := \inf_{h\in\cH} \|f-h\|_\cB.
\end{equation}
The approximation error of a set $\cF\subseteq \cB$ is the supremum approximation error of each function $f\in\cF$, i.e.
\[
\cE(\cF,\cH;\cB) := \sup_{f\in \cF} \cE(f,\cH;\cB) =\sup_{f\in \cF} \inf_{h\in\cH} \|f-h\|_\cB.
\]

Let $f\in \cB$ and $\cG,\cH\subseteq \cB$, then for any $g\in \cG$,
\[
\cE(f,\cH;\cB) = \inf_{h\in\cH} \|f-h\|_\cB \le \|f-g\|_\cB + \inf_{h\in\cH} \|g-h\|_\cB \le \|f-g\|_\cB + \cE(\cG,\cH;\cB).
\]
By taking infimum over $g\in \cG$, we get the ``triangle inequality'' for approximation error:
\[
\cE(f,\cH;\cB) \le \cE(f,\cG;\cB) + \cE(\cG,\cH;\cB).
\]

Since we will mainly characterize the approximation error by width and depth of neural networks (or by number of neurons), we define the approximation order as follows.

\begin{definition}[Order]
We say that the approximation order (by neural networks) of a function $\varphi:\bR^d \to \bR$ is at least $\alpha> 0$ if 
\[
\cE(\varphi,\cN\cN(N,L);L^{\infty}(\bR^d)) = \cO((NL)^{-\alpha}).
\]
\end{definition}
More precisely, this definition means that there exist constants $C,M>0$ such that for any positive integers $N,L\ge M$, there exists a ReLU network $\phi$ with width $N$ and depth $L$ such that
\[
\| \varphi -\phi\|_\infty \le C(NL)^{-\alpha}.
\]

\section{Approximation in shift-invariant spaces} \label{sec:shift_invariant_spaces}

Let $\varphi :\bR^d \to \bR$ be a continuous function with compact support. We consider the question that how well deep neural networks can express functions in the shift-invariant space $\cS_j(\varphi,M)$ generated by $\varphi$. More precisely, we want to estimate the size of network that is sufficient to approximate any function $g\in \cS_j(\varphi,M)$ on $(0,1)^d$ with given accuracy.

Our estimation is based on a special representation of the function $g\in \cS_j(\varphi,M)$.

\begin{lemma}\label{float_representation}
For $\Bx\in[0,1)^d$, any $g \in \cS_j(\varphi,M)$ can be written as
\[
g(\Bx) = \sum_{\Bk\in \bZ^d_\varphi} c_{m_j(\Bx)+\Bk} \varphi(r_j(\Bx)-\Bk),
\]
where the coefficients $|c_{m_j(\Bx)+\Bk}|<M$, the functions $m_j:[0,1)\to\bZ_j$ and $r_j:[0,1)\to [0,1)$ are defined by $m_j(x)=\lfloor 2^j x\rfloor$ and $r_j(x)=2^j x-m_j(x)$ and apply to $\Bx\in[0,1)^d$ coordinate-wisely, and
\[
\bZ^d_\varphi := \{\Bn\in \bZ^d: \exists \Bx\in [0,1)^d\ s.t.\ \varphi(\Bx-\Bn) \neq 0  \}.
\]
\end{lemma}
\begin{proof}
Recall that we denote $\bZ^d_j = [0,2^j-1]^d\cap \bZ^d$ and notice that $\{[0,2^{-j})^d+ 2^{-j}\Bm\}_{\Bm \in \bZ_j^d}$ is a partition of the cube $[0,1)^d$. If we denote the characteristic function of a set $A$ by $1_A$, i.e. $1_A(\Bx)=1$ if $\Bx\in A$ and $1_A(\Bx)=0$ otherwise, then for $\Bx\in[0,1)^d$,
\[
\sum_{\Bm \in \bZ_j^d} 1_{[0,2^{-j})^d+ 2^{-j}\Bm}(\Bx) = 1.
\]
For any $g \in \cS_j(\varphi,M)$ of the form (\ref{summation}), one has
\begin{align*}
g(\Bx) 
&= \sum_{\Bn\in\bZ^d} c_\Bn \varphi(2^j\Bx-\Bn) \\
&=  \sum_{\Bm \in \bZ_j^d} \sum_{\Bn\in \bZ^d} c_{\Bn} \varphi(2^j\Bx-\Bn) \cdot 1_{[0,2^{-j})^d+ 2^{-j}\Bm}(\Bx) \\
&= \sum_{\Bm \in \bZ_j^d} \sum_{\Bk\in \bZ^d_\varphi} c_{\Bm+\Bk} \varphi(2^j\Bx-\Bm-\Bk) \cdot 1_{[0,2^{-j})^d+ 2^{-j}\Bm}(\Bx).
\end{align*}
To see the last equality, notice that for each $\Bm\in \bZ^d_j$, $\Bx\in [0,2^{-j})^d+ 2^{-j}\Bm$ if and only if $2^j \Bx-\Bm\in [0,1)^d$. If we denote $\Bk:=\Bn-\Bm$, then $\varphi(2^j\Bx-\Bn)=\varphi(2^j\Bx-\Bm-\Bk)$ is a nonzero function of $\Bx$ if and only if $\Bk\in \bZ^d_\varphi$ by the definition of $\bZ^d_\varphi$. Hence the last equality holds.

Observing that $1_{[0,2^{-j})^d+ 2^{-j}\Bm}(\Bx)\neq 0$ if and only if $\Bm = m_j(\Bx)$, we have
$$
g(\Bx) = \sum_{\Bk\in \bZ^d_\varphi} c_{m_j(\Bx)+\Bk} \varphi(2^j \Bx- m_j(\Bx)-\Bk).
$$
Finally, using $r_j(\Bx) = 2^j \Bx- m_j(\Bx)$, we get the desired representation.
\end{proof}

Notice that $m_j(x)$ and $r_j(x)$ are just the integer part and fractional part of $2^j x$. They can be represented in binary forms. Let the binary representation of $x\in[0,1)$ be
\[
x=\sum_{l=1}^{\infty} 2^{-l} x_{l} = \Bin 0.x_{1}x_{2}\cdots,
\]
with $x_{l}\in \{0,1\}$. Then, by straightforward calculation,
\begin{equation}
\begin{aligned}
m_j(x) &= 2^{j-1}x_{1} + 2^{j-2}x_{2} + \cdots + 2^0x_{j}, \\
r_j(x) &=2^j x-m_j(x) = \Bin 0.x_{j+1} x_{j+2} \cdots.
\end{aligned}
\end{equation}
So $m_j(x)$ and $r_j(x)$ can be computed if we can extract the first $j$ bits of $x$, which can be done using the bit extraction technique (see section \ref{sec:bit_extraction}).

Now, suppose we can construct a network $\phi_0$ to approximate the generating function $\varphi$ with given accuracy: $\|\varphi - \phi_0\|_\infty \le \epsilon$. According to Lemma \ref{float_representation}, we can approximate $g \in \cS_j(\varphi,M)$ by concatenating $C_\varphi :=|\bZ^d_\varphi|$ sub-networks:
\[
g(\Bx) \approx \sum_{\Bk\in \bZ^d_\varphi} c_{m_j(\Bx)+\Bk} \phi_0(r_j(\Bx)-\Bk).
\]
To approximate each term, we can first extract the location information $\Bx \mapsto (m_j(\Bx),r_j(\Bx))$ using bit extraction. Then, for fixed $\Bk$, the coefficient $c_{m_j(\Bx)+\Bk} = c_\Bk(m_j(\Bx))$ can be regard as a function of $m_j(\Bx) \in \bZ_j^d$. Therefore, approximating the coefficient function $c_\Bk(m_j(\Bx))$ is equivalent to fit $\cO(|\bZ_j^d|) = \cO(2^{jd})$ samples, which can be done using $\cO(2^{jd/2})$ neurons by bit-extraction technique (see Lemma \ref{interpolate binary}). Thus, we need  $\cO(C_\varphi 2^{jd/2})$ neurons to approximate $g$ in general. 

Alternative to use the representation in Lemma \ref{float_representation}, one can approximate $g$ by computing each term in (\ref{summation}) directly.
This straightforward approach is used in \citet{shaham2018provable}, which constructs a wavelet series using a network of depth 4. Similar ideas appear in \citet{yarotsky2017error,petersen2018optimal,elbrachter2021deep,bolcskei2019optimal}. However, the size of neural networks constructed in this approach is larger than ours.
One can show that, for $\Bx \in [0,1)^d$, the non-zero terms in the summation (\ref{summation}) are those for $\Bn\in \bZ^d_\varphi + \bZ^d_j$. Since each term is approximated by one sub-network, it requires totally $\cO(C_\varphi 2^{jd})$ sub-networks to approximate $g$ in general, which needs $\cO(C_\varphi 2^{jd})$ neurons.

For our construction of ReLU neural networks, the main difficulty is that the function $m_j$ is discontinuous, hence it can not be implemented by ReLU neural networks exactly. To overcome this, we first consider the approximation on $Q(j,\delta,d)$ defined in (\ref{Q(j,delta,d)}), where we can compute $m_j(\Bx)$ and $r_j(\Bx)$ using the binary representation of $\Bx$ and the bit extraction technique. Combined with the data fitting results of deep neural networks, we can then approximate $g$ on $Q(j,\delta,d)$ to any prescribed accuracy. The approximation result is summarized in the following theorem. It also gives explicitly the required size of the network in our construction. The detailed proof is deferred to section \ref{sec:proof_Q}.

\begin{theorem}[Approximation on $Q(j,\delta,d)$] \label{approximation on Q}
Given any $j\in \bN$, $0<\delta<2^{-j}$ and $0<\epsilon<1$. Assume that $\varphi:\bR^d\to \bR$ is a continuous function with compact support and there exists a ReLU network $\phi_0$ with width $N_\varphi(\epsilon)$ and depth $L_\varphi(\epsilon)$ such that
\[
\| \varphi - \phi_0\|_\infty \le \| \varphi\|_\infty \epsilon.
\]
Then for any $g\in \cS_j(\varphi,M)$ and any $r,s,\tilde{r},\tilde{s} \in \bN$ with $2(s+r)\ge dj$ and $\tilde{r}\tilde{s}\ge \lceil \log_2(1/\epsilon)\rceil +1$, there exists a ReLU network $\phi$ with width $C_\varphi(\max\{7d\tilde{r} 2^r,N_\varphi(\epsilon)\}+4d)$ and depth $14 \tilde{s} 2^s + L_\varphi(\epsilon)$ such that for any $\Bx \in Q(j,\delta,d)$,
\[
|g(\Bx) - \phi(\Bx) | \le 3C_\varphi M\|\varphi\|_\infty \epsilon.
\]
\end{theorem}

To estimate the uniform approximation error, we will use the ``horizontal shift'' method proposed in \citet{lu2020deep}. The key idea is to approximate the target function on several domain, which have similar structure as $Q(j,\delta,d)$, that cover $[0,1]^d$ and then use the middle function $\mid(\cdot,\cdot,\cdot)$ to compute the final approximation, where $\mid(\cdot,\cdot,\cdot)$ is a function that return the middle value of the three inputs. Specifically, for each $\Bx \in [0,1]^d$, we compute three approximation of $g(\Bx)$. If at least two of these approximation have the desired accuracy, then their middle value also has the same accuracy. Using this fact, we get the following uniform approximation result.

\begin{theorem}[Uniform approximation]\label{uniform approximation}
Under the assumption of Theorem \ref{approximation on Q}, for any $g\in \cS_j(\varphi,M)$ and any $r,s,\tilde{r},\tilde{s}\in \bN$ with $2(s+r)\ge dj$ and $\tilde{r}\tilde{s}\ge \lceil \log_2(1/\epsilon)\rceil +1$, there exists a ReLU network $\phi$ with width $3^{d} \cdot 2C_\varphi(\max\{7d\tilde{r} 2^r,N_\varphi(\epsilon)\}+4d)$ and depth $14 \tilde{s} 2^s + L_\varphi(\epsilon)+2d$ such that
\[
\| g - \phi \|_{L^\infty([0,1]^d)} \le 6 C_\varphi M\|\varphi\|_\infty \epsilon.
\]
\end{theorem}

Before preceding, we would like to give a remark and a corollary on these theorems.

\begin{remark}
\textnormal{
Guaranteed by universality theorems \citep{pinkus1999approximation}, there always exist neural networks that approximate $\varphi$ arbitrarily well. But the required width $N_\varphi(\epsilon)$ and depth $L_\varphi(\epsilon)$ are generally unknown, except for certain types of $\varphi$, such as piecewise polynomials.
}
\end{remark}

\begin{corollary}
Suppose $1\le p \le \infty$ and $\varphi$ satisfies the assumption of Theorem \ref{approximation on Q}. For any $g\in \cS_j(\varphi,M)$, we have the following $L^p$ approximation result: for any $\epsilon>0$ and any $r,s\in \bN$ with $2(s+r)\ge dj$, there exists a ReLU network with width $\cO(2^r \log_2(1/\epsilon)+N_\varphi(\epsilon))$ and depth $\cO(2^s \log_2(1/\epsilon)+L_\varphi(\epsilon))$ such that 
\[
\| g - \phi \|_{L^p([0,1]^d)} \le \epsilon.
\]
\end{corollary}
\begin{proof}
The $L^p$-estimations for $1\le p<\infty$ can be obtained directly from the uniform approximation in Theorem \ref{uniform approximation} or by choosing sufficiently small $\delta$ in Theorem \ref{approximation on Q} so that the measure of $[0,1]^d \setminus Q(j,\delta,d)$ is small enough. We can choose $\tilde{r},\tilde{s} \asymp \log_2(1/\epsilon)$ in these theorems to get the desired approximation result.
\end{proof}

Besides the importance and interest in its own right, the dilated shift-invariant spaces are closely related to many other types of functions. These connections can be utilized to extend the above estimations of approximation error to other functions. Specifically, let $f$ be a function that is approximated by neural networks $\cN\cN(N,L)$ on the compact set $[0,1]^d$, we aim to estimate
\[
\cE(f,\cN\cN(N,L);L^p([0,1]^d)) = \inf_{\phi \in\cN\cN(N,L)} \|f-\phi\|_{L^p([0,1]^d)},
\]
for some $1\le p\le \infty$. For an arbitrary $f$, it is in general difficult to directly construct a neural network with given size that achieves the minimal error rate. A more feasible way is to choose some function class $\cG$ as a bridge, and estimate the approximation error by the triangle inequality
\[
\cE(f,\cN\cN(N,L);L^p([0,1]^d)) \le \cE(f,\cG;L^p([0,1]^d)) + \cE(\cG,\cN\cN(N,L);L^p([0,1]^d)).
\]
Here, we choose $\cG$ to be a dilated shift-invariant space $\cS_j(\varphi,M)$. The success of this approach depends on how well we can estimate the two terms on the right hand side of the triangle inequality. The first term is well studied in the approximation theory of shift-invariant spaces, see \citep{de1994approximation,jia1993approximation,lei1997approximation,kyriazis1995approximation,jia2004approximation,jia2010approximation}. The second term $\cE(\cS_j(\varphi,M),\cN\cN(N,L);L^p([0,1]^d))$ can be estimated using our results, i.e. Theorems \ref{approximation on Q} and \ref{uniform approximation}. Generally, we have the following.

\begin{theorem}\label{order estimate}
Let $\varphi:\bR^d \to \bR$ be a continuous function with compact support and let its approximation order be at least $\alpha$. Then for any $f:\bR^d \to \bR$ satisfying 
\[
\cE(f,\cS_j(\varphi,M);L^p([0,1]^d))=\cO(2^{-\beta j}),
\]
for some $M>0$, $1\le p\le \infty$ and $\beta>0$, we have
\[
\cE(f,\cN\cN(N,L);L^p([0,1]^d)) = \cO \left(\max\left\{ (NL)^{-\alpha}, (NL/(\log_2N \log_2L))^{-2\beta/d} \right\} \right).
\]
\end{theorem}

\begin{proof}
Denote $\epsilon = 2^{-\beta j}$. By assumption, there exists $g\in \cS_j(\varphi,M)$ such that 
\[
\|f-g\|_{L^p([0,1]^d)} = \cO(2^{-\beta j})=\cO(\epsilon).
\]
Let $r,s$ be positive integers that satisfy $2(r+s)\ge dj$. Since the approximation order of $\varphi$ is $\alpha$, there exists a network $\phi_0$ with width $N_\varphi \asymp 2^{2\beta r/(d\alpha)}$ and depth $L_\varphi \asymp 2^{2\beta s/(d\alpha)}$ such that
\[
\| \varphi -\phi_0 \|_\infty = \cO((N_\varphi L_\varphi)^{-\alpha} ) = \cO(2^{-2\beta(r+s)/d}) = \cO(2^{-\beta j}) = \cO(\epsilon).
\]
Observe that $2rs= rs+rs\ge r+s\ge dj/2 \asymp \log_2 1/\epsilon$, we can choose $\tilde{r} \asymp r$ and $\tilde{s} \asymp s$ in Theorems \ref{approximation on Q} and \ref{uniform approximation}. Thus, there exists a network $\phi$ with width $N=\cO(r2^r+2^{2\beta r/(d\alpha)})$ and depth $L=\cO(s2^s+2^{2\beta s/(d\alpha)})$ such that 
\[
\| g-\phi\|_{L^p([0,1]^d)} = \cO(\epsilon).
\]
Now we consider two cases: 

Case I: if $\alpha\ge 2\beta/d$, then we have $N=\cO(r2^r)$ and $L=\cO(s2^s)$. Hence, for any $\tilde{N}=2^r$ and $\tilde{L}=2^s$, there exists a network $\phi$ with width $N \asymp \tilde{N}\log_2 \tilde{N}$ and depth $L\asymp\tilde{L}\log_2 \tilde{L}$ such that
\[
\|f-\phi\|_{L^p([0,1]^d)} = \cO(\epsilon)= \cO(2^{-\beta j}) = \cO((\tilde{N}\tilde{L})^{-2\beta/d}) =\cO((NL/(\log_2N \log_2L))^{-2\beta/d}).
\]

Case II: if $\alpha<2\beta/d$, then we have $L=\cO(2^{2\beta s/(d\alpha)})$ and $N=\cO(2^{2\beta r/(d\alpha)})$. Hence, there exists a network $\phi$ with width $N \asymp 2^{2\beta r/(d\alpha)}$ and depth $L\asymp 2^{2 \beta s/(d\alpha)}$ such that
\[
\|f-\phi\|_{L^p([0,1]^d)} = \cO(\epsilon)= \cO(2^{-\beta j}) =\cO((NL)^{-\alpha}).
\]

Combining these two cases, we finish the proof.
\end{proof}

Roughly speaking, Theorem \ref{order estimate} indicates that the approximation order of $f$ is at least $\min\{ \alpha,2\beta/d \}$ (up to some log factors), where $\alpha$ is the approximation order of $\varphi$ by neural networks and $\beta$ is the order of the linear approximation by $S_j(\varphi,M)$. In practice, we need to choose the function $\varphi$ with large order $\alpha$, that is, the function that can be well approximated by deep neural networks. 
In particular, the approximation error $\cE(f,\cS_j(\varphi,M);L^p([0,1]^d))$ can be estimated for $f$ in many classical function spaces, such as Sobolev spaces and Besov spaces. 
It will be clear in the next section that deep neural networks can approximate piecewise polynomials with exponential convergence rate, which leads to an asymptotically optimal bound for Sobolev spaces.

\section{Application to Sobolev spaces and Besov spaces} \label{sec:Sobolev_Besov}

In this section, we apply our results to the approximation in Sobolev space $W^{\mu,p}$ and Besov space $B^\mu_{p,q}$. Similar approximation bounds can be obtained for the Triebel–Lizorkin spaces $F^\mu_{p,q}$ using the same method. The approximation rates of these spaces from shift-invariant spaces have been studied extensively in the literature \citep{de1994approximation,jia1993approximation,lei1997approximation,kyriazis1995approximation,jia2004approximation,jia2010approximation}. Roughly speaking, when $\varphi$ satisfies the  Strang–Fix condition of order $k$, then the shift-invariant space $\cS_j(\varphi)$ locally contains all polynomials of order $k-1$ and the approximation error of $f\in W^{\mu,p}$ or $f\in B^\mu_{p,q}$ is $\cO(2^{-\mu j})$ if the regularity $\mu <k$.

\subsection{Approximation of Sobolev functions and Besov functions}

We follow the quasi-projection scheme in \citet{jia2004approximation,jia2010approximation}. Suppose $1\le p\le \infty$ and $1/p+1/\tilde{p}=1$. Let $\varphi\in L^p(\bR^d)$ and $\tilde{\varphi}\in L^{\tilde{p}}$ be compactly supported functions, and, for each $\Bn\in \bZ^d$, $\varphi_\Bn = \varphi(\cdot-\Bn)$ and $\tilde{\varphi}_\Bn = \tilde{\varphi}(\cdot-\Bn)$. Then we can define the quasi-projection operator
\[
\cQ f := \sum_{\Bn\in \bZ^d} \langle f,\tilde{\varphi}_\Bn \rangle \varphi_\Bn, \quad f\in L^p(\bR^d).
\]
For $h>0$, the dilated quasi-projection operator is defined as
\[
\cQ_h f(\Bx) = \sum_{\Bn\in \bZ^d} \langle f,h^{-d/\tilde{p}} \tilde{\varphi}_\Bn(\cdot/h)  \rangle h^{-d/p} \varphi_\Bn(\Bx/h), \quad \Bx \in \bR^d.
\]
Notice that if $h=2^{-j}$, $\cQ_hf$ is in the completion of the shift-invariant space $\cS_j(\varphi)$. 

If $\varphi$ satisfies the Strang-Fix condition of order $k$:
\[
\hat{\varphi}(0)\neq 0,\ \mbox{and } D^{\boldsymbol{\alpha}} \hat{\varphi}(2\Bn \pi)=0,\quad \Bn\in \bZ^d\setminus \{0\}, |\boldsymbol{\alpha}|<k,
\]
where $\hat{\varphi}(\boldsymbol{\omega}) = \int \varphi(\Bx)e^{-i\Bx \cdot \boldsymbol{\omega}}d\Bx$ is the Fourier transform of $\varphi$, then we can choose $\tilde{\varphi}$ such that the quasi-projection operator $\cQ$ has the polynomial reproduction property: $\cQ g=g$ for all polynomials $g$ with order $k-1$. The approximation error $f-\cQ f$ has been estimate in \citet{jia2004approximation,jia2010approximation} when the quasi-projection operator has the polynomial reproduction property. The following lemma is a consequence of the results.

\begin{lemma}\label{Strang-Fix}
Let $k\in \bN$, $0<\mu<k$, $1\le p,q\le \infty$ and $\cF$ be either the Sobolev space $W^{\mu,p}$ or the Besov space $B^\mu_{p,q}$. If $\varphi$ satisfies the Strang-Fix condition of order $k$, then there exists $\tilde{\varphi}$ and a constant $C>0$ such that for any $f\in \cF$,
\[
\|f-\cQ_{2^{-j}}f\|_p \le C 2^{-\mu j} \|f\|_\cF.
\]
\end{lemma}

A fundamental example that satisfies the Strang-Fix condition is the multivariate B-splines of order $k\ge 2$ defined by
\[
\cN_k^d(\Bx) := \prod_{i=1}^d \cN_k(x_i), \quad \Bx=(x_1,\dots,x_d)\in \bR^d,
\]
where the univariate cardinal B-spline $\cN_k$ of order $k$ is given by
\[
\cN_k(x) := \frac{1}{(k-1)!} \sum_{l=0}^{k}(-1)^l \binom{k}{l}\sigma(x-l)^{k-1}, \quad x\in \bR.
\]
It is well known that $\|\cN_k\|_\infty =1$ and the support of $\cN_k$ is $[0,k]$. Alternatively, the B-spline $\cN_k$ can be defined inductively by the convolution $\cN_k = \cN_{k-1}*\cN_1$ where $\cN_1(x)=1$ for $x\in[0,1]$ and $\cN_1(x)=0$ otherwise. Hence, the Fourier transform of $\cN_k$ is $\widehat{\cN_k}(\omega) = \left( \frac{1-e^{- i\omega}}{ i \omega} \right)^k$. The relation of B-splines approximation and Besov spaces is discussed in \citet{devore1988interpolation}.

The following lemma gives the approximation order of $\cN_k$ by deep neural networks. 

\begin{lemma} \label{Bspline app}
For any $N,L,k\in \bN$ with $k\ge 3$, there exists a ReLU network $\phi$ with width $d(k+1)(9(N+1)+k)$ and depth $7(k^2+d^2)L$ such that
\[
\|\cN_k^d-\phi\|_\infty \le 9d \frac{(2k+2)^k}{(k-1)!} (N+1)^{-7(k-1)L} + 9(d-1)(N+1)^{-7dL}.
\]
\end{lemma}

Given any $k\ge 3$, this lemma implies that 
\[
\cE(\cN_k^d, \cN\cN(N,L);L^\infty(\bR^d)) = \cO(N^{-\cO(L)}).
\]
Hence, the approximation order of $\cN_k^d$ can be chosen to be any $\alpha>0$. Theorem \ref{order estimate} and Lemma \ref{Strang-Fix} imply that the approximation error of any $f\in W^{\mu,p}$ or $f\in B^\mu_{p,q}$ is 
\[
\cO((NL/(\log_2N \log_2L))^{-2\mu/d}).
\]
A more detailed analysis reveals that this bound is uniform for the unit ball of the spaces. We summarize the results in the following theorem.

\begin{theorem}\label{Sobolev NL bound}
Let $\cF$ be either the unit ball of Sobolev space $W^{\mu,p}$ or the Besov space $B^\mu_{p,q}$. We have the following estimate of the approximation error
\[
\cE(\cF,\cN\cN(N,L);L^p([0,1]^d)) = \cO \left( (NL/(\log_2N \log_2L))^{-2\mu/d} \right).
\]
\end{theorem} 
\begin{proof}
Let $f\in \cF$ and $k>\mu$, then by Lemma \ref{Strang-Fix}, $\|f-\cQ_{2^{-j}}f\|_p\le C2^{-\mu j}\|f\|_\cF$, where $\cQ_{2^{-j}}f = \sum_{\Bn\in \bZ^d} c_\Bn(f) \cN_k^d(2^j\cdot-\Bn)$ is a B-spline series. It can be shown that the coefficients of a B-spline series is bounded by the $L^p$ norm of the series: $|c_\Bn(f)|\le C2^{dj/p}\|\cQ_{2^{-j}}(f)\|_p$ (See, for example, \citep[Chapter 5.4]{devore1993constructive} and \citep[Lemma 4.1]{devore1988interpolation}). Hence, $|c_\Bn(f)|\le C2^{dj/p} \|f\|_\cF\le C2^{dj/p}$, which implies $\cQ_{2^{-j}}f\in \cS_j(\cN_k^d,M)$ with $M\le C2^{dj/p}$.

Let $r,s\in \bN$ satisfy $2(r+s)\ge dj$, denote $\widetilde{N}=2^r$, $\widetilde{L}=2^s$ and choose $\epsilon= 2^{-\mu j-dj/p}$. By Lemma \ref{Bspline app}, if $N,L$ are sufficiently large, $\cE(\cN_k^d, \cN\cN(N,L),L^\infty(\bR^d))\le C_{k,d}N^{-7L}\le C_{k,d}(NL)^{-\alpha}$, where we choose $\alpha=2+2\mu/d$. Thus, there exists a network $\phi_0$ with width $\cO(\widetilde{N})$ and depth $\cO(\widetilde{L})$ such that
\[
\|\cN_k^d-\phi_0\|_\infty \le (\widetilde{N}\widetilde{L})^{-\alpha} \le 2^{-dj\alpha/2} = 2^{-\mu j -dj} \le \epsilon.
\]
Since $2rs= rs+rs\ge r+s\ge dj/2 \asymp \log_2 1/\epsilon$, we can choose $\tilde{r} \asymp r$ and $\tilde{s} \asymp s$ in Theorem \ref{approximation on Q} and Theorem \ref{uniform approximation}. Therefore, since $\cQ_{2^{-j}}f\in \cS_j(\cN_k^d,M)$, there exists a network $\phi$ with width $\cO(\widetilde{N}\log_2\widetilde{N})$ and depth $\cO(\widetilde{L}\log_2\widetilde{L})$ such that
\[
\| \cQ_{2^{-j}}f - \phi \|_{L^p([0,1]^d)} \le CM\epsilon \le C2^{-\mu j}.
\]
The triangle inequality gives
\[
\| f - \phi\|_{L^p([0,1]^d)} \le \cO(2^{-\mu j }) = \cO((\widetilde{N}\widetilde{L})^{-2\mu/d}).
\]

Finally, let $N= \tilde{N}\log_2\tilde{N}$ and $L = \tilde{L}\log_2\tilde{L}$, we have
\[
\cE(f,\cN\cN(N,L);L^p([0,1]^d)) = \cO((\widetilde{N}\widetilde{L})^{-2\mu/d}) = \cO \left( (NL/(\log_2N \log_2L))^{-2\mu/d} \right).
\]
Since the bound is uniform for all $f\in \cF$, we finish the proof.
\end{proof}

So far, we characterize the approximation error by the number of neurons $NL$, we can also estimate the error by the number of weights. To see this, let the width $N$ be sufficiently large and fixed, then the number of weights $W\asymp N^2L\asymp L$ and we have 
\[
\cE(\cF,\cN\cN(N,L);L^p([0,1]^d)) = \cO \left( (W/\log_2 W)^{-2\mu/d} \right).
\]
Note that similar bounds have been obtained in \citet{yarotsky2019phase} and \citet{lu2020deep} for H\"older spaces. The paper \citep{suzuki2018adaptivity} also studies the approximation in Besov spaces, but they only get the bound $\cO(W^{-\mu/d})$.

\subsection{Optimality for Sobolev spaces}\label{sec:optimality}
We consider the optimality of the upper bounds we have derived for the unit ball $\cF$ of Sobolev spaces $W^{k,p}$. The main idea is to find the connection between the approximation accuracy and the Pseudo-dimension (or VC-dimension) of neural networks. Let us first introduce some results of Pseudo-dimension.

\begin{definition}[Pseudo-dimension]
Let $\cH$ be a class of real-valued functions defined on $\Omega$. The Pseudo-dimension of $\cH$, denoted by $\Pdim(\cH)$, is the largest integer of $N$ for which there exist points $x_1,\dots,x_N \in \Omega$ and constants $c_1,\dots,c_N\in \bR$ such that 
\[
|\{ \sgn(h(x_1)-c_1),\dots,\sgn(h(x_N)-c_N): h\in \cH \}| =2^N.
\]
If no such finite value exists, $\Pdim(\cH)=\infty$.
\end{definition}

There are some well-known upper bounds on Pseudo-dimension of deep ReLU networks in the literature \citep{anthony2009neural,bartlett1999almost,goldberg1995bounding,bartlett2019nearly}. We summarize two bounds in the following lemma.

\begin{lemma}\label{Pdim bound}
Consider a network architecture $\eta$ with $W$ parameters, $U$ neurons and depth $L$. Let $\cH_\eta$ be the set of functions that can be represented by such architecture with ReLU activation. Then there exists constants $C_1,C_2> 0$ such that
\[
\Pdim(\cH_\eta) \le C_1 W^2 \quad \mbox{and} \quad \Pdim(\cH_\eta) \le C_2 WL\log_2 U.
\]
\end{lemma}

Intuitively, if a function class $\cH$ can approximate a function class $\cF$ of high complexity with small precision, then $\cH$ should also have high complexity. In other words, if we use a function class $\cH$ with $\Pdim(\cH)\le n$ to approximate a complex function class, we should be able to get a lower bound of the approximation error. Mathematically, we can define a nonlinear $n$-width using Pseudo-dimension: let $\cB$ be a normed space and $\cF\subseteq \cB$, we define
\[
\rho_n (\cF,\cB):= \inf_{\cH^n} \cE(\cF,\cH^n;\cB) = \inf_{\cH^n} \sup_{f\in \cF} \inf_{h\in\cH^n} \|f-h\|_\cB,
\]
where $\cH^n$ runs over all classes in $\cB$ with $\Pdim(\cH^n)\le n$. 

We remark that the $n$-width $\rho_n$ is different from the famous continuous $n$-th width $\omega_n$ introduced by  \citet{devore1989optimal}:
\[
\omega_n(\cF,\cB):= \inf_{\Ba,M_n} \sup_{f\in \cF} \|f-M_n(\Ba(f))\|_\cB,
\]
where $\Ba:\cF\to \bR^n$ is continuous and $M_n:\bR^n\to \cF$ is any mapping. In neural network approximation, $\Ba$ maps the target function $f\in\cF$ to the parameters in neural network and $M_n$ is the realization mapping that associates the parameters to the function realized by neural network. Applying the results in \citet{devore1989optimal}, one can show that the approximation error of the unit ball of Sobolev space $W^{k,p}(\bR^d)$ is lower bounded by $cW^{-k/d}$, where $W$ is the number of parameters in the network, see \citep{yarotsky2017error,yarotsky2019phase}. However, we have obtained an upper bound $\cO((W/\log_2 W)^{-2k/d})$ for these spaces. The inconsistency is because the parameters in our construction does not continuously depend on the target function and hence it does not satisfy the requirement in the $n$-width $\omega_n$. This implies that we can get better approximation order by taking advantage of the incontinuity. 

The $n$-width $\rho_n$ was firstly introduced by  \citet{maiorov1999degree,ratsaby1997value}. They also gave upper and lower estimates of the $n$-width for Sobolev spaces. The following lemma is from \citet{maiorov1999degree}.

\begin{lemma}\label{n-width}
Let $\cF$ be the unit ball of Sobolev space $W^{k,p}(\bR^d)$ and $1\le p,q\le \infty$, then
\[
\rho_n (\cF,L^q([0,1]^d)) \ge cn^{-k/d},
\]
for some constant $c>0$ independent of $n$.
\end{lemma}

Combining Lemmas \ref{Pdim bound} and \ref{n-width}, we can give lower bound of the approximation error by ReLU neural networks. These lower bounds show that the upper bound in Theorem \ref{Sobolev NL bound} is asymptotically optimal up to a logarithm factor.

\begin{corollary}\label{lower bound}
Let $\cF$ be the unit ball of Sobolev space $W^{k,p}(\bR^d)$ and $1\le p,q\le \infty$. For the function class $\cH_\eta$ in Lemma \ref{Pdim bound}, we have
\[
\cE(\cF,\cH_\eta;L^q([0,1]^d)) \ge c_1 W^{-2k/d}\quad \mbox{and} \quad  \cE(\cF,\cH_\eta;L^q([0,1]^d)) \ge c_2 (WL\log_2 U)^{-k/d}.
\]
for some constant $c_1,c_2>0$. In particular, there exists $c>0$ such that
\[
\cE(\cF,\cN\cN(N,L);L^q([0,1]^d)) \ge c (N^2L^2\log_2 NL)^{-k/d}.
\]
\end{corollary}
\begin{proof}
We choose $n= \Pdim(\cH_\eta)$, then by the definition of the $n$-width $\rho_n (\cF,L^q([0,1]^d))$ and Lemma \ref{n-width}, 
\[
\cE(\cF,\cH_\eta;L^q([0,1]^d)) \ge \rho_n (\cF,L^q([0,1]^d)) \ge c_0n^{-k/d}.
\]
By lemma \ref{Pdim bound}, we have $n= \Pdim(\cH_\eta)\le C_1 W^2$ and $n\le C_2 WL \log_2 U$, which give the desired lower bounds for $\cE(\cF,\cH_\eta;L^q([0,1]^d))$. 

When $\cH$ is the fully connected network $\cN\cN(N,L)$, we have $W=\cO(N^2L)$ and $U=\cO(NL)$. Hence, 
\[
\cE(\cF,\cN\cN(N,L);L^q([0,1]^d)) \ge c_2 (WL\log_2 U)^{-k/d} \ge c (N^2L^2\log_2 NL)^{-k/d}. \qedhere
\]
\end{proof}

\section{Discussion}\label{sec:discussion}

In this paper, we study how well deep ReLU networks can approximate functions in dilated shift-invariant spaces. Our main theorems, Theorem \ref{approximation on Q} and \ref{uniform approximation}, give upper bounds on the approximation error of these spaces. The results can be easily applied to wavelet, which is widely used in signal processing. As an illustration, we consider a multiresolution approximation $\{V_j\}_{j\in \bZ}$ of $L^2(\bR)$, which satisfies $V_{j+1} \subseteq V_j$ for all $j\in\bZ$. And let $\psi$ and $\varphi$ be the wavelet and the scaling function that generate an orthogonal basis \citep[Chapter 7]{mallat1999wavelet}. Denote $\varphi_{j,n}(x) := 2^{-j/2} \varphi(2^{-j}x-n)$, then the orthogonal projection of $f\in L^2(\bR)$ over $V_{-j}$ is
\[
P_{V_{-j}} f = \sum_{n=-\infty}^{+\infty} \langle f, \varphi_{-j,n} \rangle \varphi_{-j,n},
\]
which has the same form of functions in the dilated shift-invariant space $\cS_j(\varphi)$. Hence, Theorem \ref{approximation on Q} and \ref{uniform approximation} can be applied to derive approximation bound of the projection $P_{V_{-j}} f$. Alternatively, we can also approximate the wavelet decomposition
\[
f = \sum_{j=-\infty}^{\infty}\sum_{n=-\infty}^{+\infty} \langle f, \psi_{j,n} \rangle \psi_{j,n},
\]
using the approximation result for $\cS_j(\psi)$. 

The abstract approximation results for shift-invariant spaces can also be applied to study the approximation of classical smooth function spaces by deep neural networks, which has received much attention in recent years. When the approximation error is measured by the number parameters $W$, the seminal work of \citet{yarotsky2017error} obtained approximation bound $\cO(W^{-s/d})$ for the Sobolev spaces $W^{s,\infty}$, ignoring the logarithmic factors. The recent works  \citep{yarotsky2018optimal,yarotsky2019phase} improved the upper bound to $\cO(W^{-2s/d})$. In contrast, if the error is measured by the number of neurons $U \asymp NL$, \citet{shen2019deep,lu2020deep} showed the bound $\cO \left( U^{-2s/d} \right)$ for smooth function class $C^s([0,1]^d)$. All these results are derived through approximating local Taylor expansions by neural networks. In this paper, we take a multiresolution approximation point of view. By choosing the B-spline as the generating function $\varphi$ of the shift-invariant space $\cS_j(\varphi)$, we can recover all the existing bounds and generalize them to the Besov spaces $B^\mu_{p,q}$. Our result improves the existing bounds for Besov spaces obtained by \citet{suzuki2018adaptivity}. We also prove lower bounds of the approximation error in $L^p$-norm ($1\le p\le \infty$), which show the optimality of the upper bounds. As far as we know, only $L^\infty$ lower bounds for neural network approximation are obtained in the literature.

Although the lower bounds in Corollary \ref{lower bound} are proved for ReLU networks, similar lower bounds can be derived for piecewise polynomial activation functions using the same argument and the upper bounds of Pseudo-dimension for such activation functions in \citet{bartlett2019nearly}. However, for more complicated activation functions, this kind of lower bound may not exist. For example, \citet{maiorov1999lower} showed that there exists an analytic, strictly increasing, and sigmoidal activation function such that any continuous function on $[0,1]^d$ can be uniformly approximated to within any error by a neural network with width $6d+3$ and depth $3$. In other words, we can approximate any continuous function using a network of \emph{fixed} finite size with this activation function. However, by Lemma \ref{n-width}, the function class generated by this network has infinite Pseudo-dimension, which is due to the high ``complexity'' of the activation function.

\section{Proof of Theorem \ref{approximation on Q}} \label{sec:proof_Q}

Without loss of generality, we can assume that $M=1$ and $\|\varphi\|_\infty= 1$. By Lemma \ref{float_representation}, for $\Bx\in[0,1)^d$,
\[
g(\Bx) = \sum_{\Bn\in\bZ^d} c_{\Bn} \varphi(2^j\Bx-\Bn) = \sum_{\Bk\in \bZ^d_\varphi} c_{m_j(\Bx)+\Bk} \varphi(r_j(\Bx)-\Bk),
\]
where $m_j$ and $r_j$ are applied coordinate-wisely to $\Bx=(x_1,\dots,x_d)\in [0,1)^d$ and
\begin{align*}
m_j(x_i) &= 2^{j-1}x_{i,1} + 2^{j-2}x_{i,2} + \cdots + 2^0x_{i,j}, \\
r_j(x_i) &=2^j x-m_j(x_i) = \Bin 0.x_{i,j+1} x_{i,j+2} \cdots,
\end{align*}
if $x_i = \Bin 0.x_{i,1}x_{i,2}\cdots$ is the binary representation of $x_i\in[0,1)$.

For any fixed $\Bk\in \bZ^d$, we are going to construct a network that approximates the function
\[
\Bx \mapsto c_{m_j(\Bx)+\Bk} \varphi(r_j(\Bx)-\Bk), \quad \Bx\in Q(j,\delta,d).
\]
We can summarize the result as follows.
\begin{proposition}\label{key approximation}
For any fixed $j\in \bN$ and $\Bk\in \bZ^d$, there exists a network $\phi^{(\Bk)}$ with width $\max\{7d\tilde{r} 2^r,N_\varphi(\epsilon)\}+4d$ and depth $14 \tilde{s} 2^s + L_\varphi(\epsilon)$ such that for any $\Bx\in Q(j,\delta,d)$,
\[
|c_{m_j(\Bx)+\Bk} \varphi(r_j(\Bx)-\Bk) - \phi^{(\Bk)}(\Bx)|\le 3\epsilon.
\]
\end{proposition}

Assume that Proposition \ref{key approximation} is true. We can construct the desired function $\phi$ by
\[
\phi(\Bx) = \sum_{\Bk\in \bZ^d_\varphi} \phi^{(\Bk)}(\Bx),
\]
which can be computed by $C_\varphi$ parallel sub-networks $\phi^{(\Bk)}$. Since $\phi(\Bx)$ is a linear combination of $\phi^{(\Bk)}(\Bx)$, the required depth is $14 \tilde{s} 2^s + L_\varphi(\epsilon)$ and the required width is at most $C_\varphi(\max\{7d\tilde{r} 2^r,N_\varphi(\epsilon)\}+4d)$. The approximation error is
\[
\left| \sum_{\Bn\in\bZ^d} c_{\Bn} \varphi(2^j\Bx-\Bn) - \phi(\Bx) \right| \le \sum_{\Bk\in \bZ^d_\varphi} |c_{m_j(\Bx)+\Bk} \varphi(r_j(\Bx)-\Bk) - \phi^{(\Bk)}(\Bx)| \le 3C_\varphi \epsilon.
\]

It remains to prove Proposition \ref{key approximation}. The key idea is as follows. Since $|c_{\Bm}|< 1$, we let $b_i (\Bm)\in \{0,1\}$ be the $i$-bit of $c_{\Bm} /2+ 1/2 \in [0,1)$. Thus, we have the binary representation
\begin{equation}\label{definition of b}
c_{\Bm} = \sum_{i=1}^\infty 2^{1-i} b_i (\Bm) -1.
\end{equation}
As a consequence, we have
\[
c_{m_j(\Bx)+\Bk} \varphi(r_j(\Bx)-\Bk) = \sum_{i=1}^\infty 2^{1-i} b_i (m_j(\Bx)+\Bk) \varphi(r_j(\Bx)-\Bk) -\varphi(r_j(\Bx)-\Bk).
\]
We will construct a neural network that approximates the truncation
\[
\sum_{i=1}^{\lceil \log_2(1/\epsilon)\rceil +1} 2^{1-i} b_i (m_j(\Bx)+\Bk) \varphi(r_j(\Bx)-\Bk) -\varphi(r_j(\Bx)-\Bk), \quad \Bx \in Q(j,\delta,d).
\]
The construction can be divided into two parts:
\begin{enumerate}
	\item For each $j\in \bN$, construct a neural network to compute $\Bx \mapsto (m_j(\Bx), r_j(\Bx))$. This can be done by the bit extraction technique.
	
	\item For each $i,j\in \bN$, construct a neural network to compute $\Bm \in \bZ_j^d \mapsto b_i (\Bm)$, which is equivalent to interpolate $2^{dj}$ samples $(\Bm,b_i (\Bm))$.
	
\end{enumerate}
We gather the necessary results in the following two subsections and give a proof of Proposition \ref{key approximation} in subsection \ref{archetecture}.

\subsection{Bit extraction}\label{sec:bit_extraction}

In order to compute $\Bin 0.x_1x_2\cdots \mapsto (x_1,\dots,x_r)$, we need to use the bit extraction technique in \citet{bartlett1999almost,bartlett2019nearly}. Let us first introduce the basic lemma that extract $r$ bits using a shallow network.

\begin{lemma}\label{bit extraction}
Given $j\in \bN$ and $0<\delta<2^{-j}$. For any positive integer $r\le j$, there exists a network $\phi_r$ with width $2^{r+1}+1$ and depth $3$ such that 
\[
\phi_r(x) = (x_1,\dots,x_r,\Bin 0.x_{r+1}x_{r+2}\cdots), \quad \forall x= \Bin 0.x_1x_2\cdots \in Q(j,\delta,1).
\]
\end{lemma}

\begin{proof}
We follow the construction in \citet{bartlett2019nearly}. For any $a\le b$, observe that the function
\[
f_{[a,b]} (x) := \sigma(1-\sigma(a/\delta-x/\delta)) + \sigma(1-\sigma(x/\delta - b/\delta)) - 1
\]
satisfies $f_{[a,b]}(x) =1$  for $x\in[a,b]$, and $f_{[a,b]}(x) = 0$ for $x\notin (a-\delta,b+\delta)$, and $f_{[a,b]}(x)\in [0,1]$ for all $x$. So, we can use $f_{[a,b]}$ to approximate the indicator function of $[a,b]$, to precision $\delta$. Note that $f_{[a,b]}$ can be implemented by a ReLU network with width $2$ and depth $3$. 

Since $x_1,\dots,x_r$ can be computed by adding the corresponding indicator functions of $[k2^{-r},(k+1)2^{-r}]$, $0\le k\le 2^r-1$, we can compute $x_1,\dots,x_r$ using $2^r$ parallel networks
\[
f_{[1-2^{-r},1]},\quad f_{[k2^{-r},(k+1)2^{-r}-\delta]},\quad k=0,\dots,2^r-2.
\]
Observe that 
\[
\Bin 0.x_{r+1}x_{r+2}\cdots = 2^r x-\sum_{i=1}^{r}2^{r-i}x_i,
\]
which is a linear combination of $x,x_1,\dots,x_r$. There exists a network $\phi_r$ with width $2^{r+1}+1$ and depth $3$ such that $\phi_r(x) = (x_1,\dots,x_r,\Bin 0.x_{r+1}x_{r+2}\cdots)$ for $x\in Q(j,\delta,1)$. (We can use one neuron in each hidden layer to 'remember' the input $x$. Since $\phi_r(x)$ is a linear transform of $x$ and outputs of the parallel sub-networks, we do not need extra layer to compute summation.)
\end{proof}

Note that the function $\Bin 0.x_1x_2\cdots \mapsto (x_1,\dots,x_r)$ is not continuous, while every ReLU network function is continuous. So, we cannot implement the bit extraction on the whole set $[0,1]$. This is why we restrict ourselves to $Q(j,\delta,1)$.

The next lemma is an extension of Lemma \ref{bit extraction}. It will be used to extract the location information $(m_j(\Bx),r_j(\Bx))$.

\begin{lemma}\label{bit extraction sum}
Given $r,j\in \bN$ and $0<\delta<2^{-j}$. For any integer $k\le j$, there exists a ReLU network $\phi$ with width $2^{r+1} +3$ and depth $2\lceil j/r \rceil +1$ such that 
\[
\phi(x) = \left(\sum_{i=1}^k 2^{j-i}x_i,\sum_{i=k+1}^j 2^{j-i}x_i, \Bin 0.x_{j+1}x_{j+2}\cdots \right), \quad \forall x= \Bin 0.x_1x_2\cdots \in Q(j,\delta,1).
\]
\end{lemma}

\begin{proof}
Without loss of generality, we can assume $r\le j$. By Lemma \ref{bit extraction}, there exists a network $\phi_r$ with width $2^{r+1}+1$ and depth $3$ such that $\phi_r(x) = (x_1,\dots,x_r,\Bin 0.x_{r+1}x_{r+2}\cdots)$. Observe that any summation $\sum_{i=1}^k 2^{j-i}x_i$ and $\sum_{i=k+1}^r 2^{j-i}x_i$ with $k\le r$ are linear combinations of outputs of $\phi_r$. We can compute them by a network having the same size as $\phi_r$. If $k>r$, we compute $\sum_{i=1}^r 2^{j-i}x_i$ as intermediate result. Then, by applying another $\phi_r$ to $\Bin 0.x_{r+1}x_{r+2}\cdots$, we can extract the next $r$ bits $x_{r+1},\dots, x_{2r}$, and compute $\Bin 0.x_{2r+1}x_{2r+2}\cdots$. Again, any summation  $\sum_{i=1}^{k} 2^{j-i}x_i$ and $\sum_{i=k+1}^{2r} 2^{j-i}x_i$ with $k\le 2r$ are linear combinations of the outputs. If $k>2r$, we compute $\sum_{i=1}^{2r} 2^{j-i}x_i$ as intermediate result. Continuing this strategy, after we extract $\lfloor j/r \rfloor r$ bits, we can use $\phi_{j-\lfloor j/r \rfloor r}$ to extract the rest bits. Using this construction, we can compute the required function $\phi$ by a network with width at most $2^{r+1} +3$ and depth at most $2\lceil j/r \rceil +1$. (Two neurons in each hidden layer are used to 'remember' the intermediate computation.)
\end{proof}

The following lemma shows how to extract a specific bit.

\begin{lemma}\label{extract one bit}
For any $r,K\in \bN$ with $r\le K$, there exists a ReLU network $\phi$ with width $2^{r+1}+3$ and depth $4\lceil K/r \rceil +1$ such that for any $x= \Bin 0.x_1x_2\cdots x_K$ and positive integer $k\le K$, we have $\phi(x,k) = x_k$.
\end{lemma}
\begin{proof}
Let $\delta_{ki} =0$ if $k\neq i$ and $\delta_{ki} =1$ if $k=i$. Observe that 
\[
\delta_{ki} = \sigma(k-i+1) + \sigma(k-i-1) - 2\sigma(k-i),
\]
and $t_1t_2 = \sigma(t_1+t_2-1)$ for any $t_1,t_2\in \{0,1\}$. We have the expression
\[
x_k = \sum_{i=1}^K \delta_{ki}x_i = \sum_{i=1}^K \sigma \left( \sigma(k-i+1) + \sigma(k-i-1) - 2\sigma(k-i) +x_i -1 \right).
\]

By Lemma \ref{bit extraction}, there exists a ReLU network $\phi_r$ with width $2^{r+1}+1$ and depth $3$ such that $\phi_r(x) = (x_1,\dots,x_r,\Bin 0.x_{r+1}x_{r+2}\cdots x_K)$. Hence, the function 
\[
\tilde{\phi}_r(x,k) = \left(\Bin 0.x_{r+1}x_{r+2}\cdots x_K, k,  \sum_{j=1}^r \delta_{kj}x_j \right)
\]
is a network with width at most $\max\{2^{r+1},4r\}+3= 2^{r+1}+3$ and depth $5$. Applying Lemma \ref{bit extraction} to the first output $\Bin 0.x_{r+1}x_{r+2}\cdots x_K$ and preserving the last output $(k,\sum_{j=1}^r \delta_{kj}x_j)$, we can implement
\[
\tilde{\phi}_{2r}(x,k) = \left(\Bin 0.x_{2r+1}x_{2r+2}\cdots x_K, k,  \sum_{j=1}^{2r} \delta_{kj}x_j \right)
\]
by a network with width $2^{r+1}+3$ and depth $9$. Using this construction iteratively, we can implement the required function $\phi(x,k) = x_k = \sum_{j=1}^K \delta_{kj}x_j$ by a network with width at most $2^{r+1}+3$ and depth $4\lceil K/r \rceil +1$.
\end{proof}

\subsection{Interpolation}

Given an arbitrary sample set $(x_i,y_i)$, $i=1,\dots,M$, we want to find a network $\phi$ with certain architecture to interpolate the data: $\phi(x_i)=y_i$. This problem has been studied in many papers \citep{yun2019small,shen2019nonlinear,vershynin2020memory}. Roughly speaking, the number of samples that a network can interpolate is in the order of the number of parameters.

The following lemma is a combination of Proposition 2.1 and 2,2 in \citet{shen2019nonlinear}.

\begin{lemma}\label{interpolate real}
For any $N,L\in \bN$, given $N^2L$ samples $(x_i,y_i)$, $i=1,\dots,N^2L$, with distinct $x_i\in \bR^d$ and $y_i\ge 0$. There exists a ReLU network $\phi$ with  width $4N+4$ and depth $L+2$ such that $\phi(x_i) =y_i$ for $i=1,\dots,N^2L$.
\end{lemma}

We can also give an upper bound of the interpolation capacity of a given network architecture.

\begin{proposition}\label{Hausdorff estimate}
Let $\cH=\{\phi_{\theta} : \bR^{d_{in}}\to \bR^{d_{out}}\}$ be the class of functions that can be represented by a ReLU network with architecture of $W$ parameters $\theta$. If for any $M$ samples $(x_i,y_i)$ with distinct $x_i\in \bR^{d_{in}}$ and $y_i\in \bR^{d_{out}}$, there exists $\theta$ such that $\phi_\theta(x_i) =y_i$ for $i=1,\dots,M$, then $W\ge Md_{out}$.
\end{proposition}
\begin{proof}
Choose any $M$ distinct points $\{x_i \}_{i=1}^N \subseteq \bR^{d_{in}}$. We consider the function $F:\bR^W \to \bR^{Md_{out}}$ defined by
\[
F(\theta) = (\phi_\theta(x_1),\dots,\phi_\theta(x_M)).
\]
By assumption, $F$ is surjective. Since $F$ is a continuous piecewise multivariate polynomial, it is Lipschitz on any closed ball. Therefore, the Hausdorff dimension of the image under $F$ of any closed ball is at most $W$ \citep[Theorem 2.8]{evans2015measure}. Since $\bR^{Md_{out}} = F(\bR^W)$ is a countable union of images of closed balls, its Hausdorff dimension is at most $W$. Hence, $Md_{out}\le W$.
\end{proof}

This proposition shows that a ReLU network with width $N$ and depth $L$ can interpolate at most $\cO (N^2L)$ samples, which implies the construction in Lemma \ref{interpolate real} is asymptotically optimal. However, if we only consider Boolean output, we can construct a network with width $\cO(N)$ and depth $\cO(L)$ to interpolate $N^2L^2$ well-spacing samples. The construction is based on the bit extraction Lemma \ref{extract one bit}.

\begin{lemma}\label{interpolate binary}
Let $N,L\in \bN$. Given any $N^2L^2$ samples $\{(x_i,k,y_{i,k}) :i=1,\dots,N^2L,\  k=1,\dots,L \}$, where $x_i\in \bR^d$ are distinct and $y_{i,k}\in \{0,1\}$. There exists a ReLU network $\phi$ with width $4N+5$ and depth $5L+2$ such that $\phi(x_i,k) =y_{i,k}$ for $i=1,\dots,N^2L$ and $k=1,\dots,L$.
\end{lemma}

\begin{proof}
For any $i=1,\dots,N^2L$, denote $y_i = \Bin 0.y_{i,1}y_{i,2}\cdots y_{i,L} \in [0,1]$. Considering the $N^2L$ samples $(x_i,y_i)$, by Lemma \ref{interpolate real}, there exists a network $\phi_1$ with width $4N+4$ and depth $L+2$ such that $\phi_1(x_i) =y_i$ for $i=1,\dots,N^2L$. 

By Lemma \ref{extract one bit}, there exists a network $\phi_2$ with width $7$ and depth $4L+1$ such that $\phi_2(y_i,k) = y_{i,k}$ for any $i=1,\dots,N^2L$ and $k=1,\dots,L$. Hence, the function $\phi(x,k) = \phi_2(\phi_1(x),k)$ can be implemented by a network with width $4N+5$ and depth $5L+2$.
\end{proof}

The pseudo-dimension of a network with width $N$ and depth $L$ is $\cO(N^2L^2\log_2(NL))$, which means $\cN\cN(N,L)$ can interpolates at most $\cO(N^2L^2\log_2(NL))$ samples with Boolean outputs. Hence, the construction in Lemma \ref{interpolate binary} is optimal up to a logarithm factor. But we require that the samples are well-spacing in the lemma.

\subsection{Proof of Proposition \ref{key approximation}}\label{archetecture}

Now, we are ready to prove Proposition \ref{key approximation}. For simplicity, we only consider the case $\Bk=(0,\dots,0)$, the following construction can be easily applied to general $\Bk\in \bZ^d$. 

Recall that
\[
c_{m_j(\Bx)} \varphi(r_j(\Bx)) = \sum_{i=1}^\infty 2^{1-i} b_i (m_j(\Bx)) \varphi(r_j(\Bx)) -\varphi(r_j(\Bx)),
\]
where $b_i (\Bm)\in \{0,1\}$ is the $i$-bit of $c_{\Bm} /2+ 1/2 \in [0,1)$. For any fixed $i, j\in \bN$, we first construct a network to approximate
\[
2^{1-i} b_i (m_j(\Bx)) \varphi(r_j(\Bx)).
\]

For any $r,s\in \bN$ with $2(r+s)\ge jd$, by Lemma \ref{bit extraction sum}, there exist ReLU networks $h_m:\bR \to \bR^{3}$, $1\le m\le d$, with width $2^{r+1}+3$ and depth $2\lceil j/r \rceil +1$ such that for any $x_m = \Bin 0.x_{m,1}x_{m,2} \cdots \in Q(j,\delta,1)$, 
\[
h_m(x_m) = \left(\sum_{\ell=1}^{k_m} 2^{j-\ell}x_{m,\ell},\sum_{\ell=k_m+1}^j 2^{j-\ell}x_{m,\ell}, \Bin 0.x_{m,j+1}x_{m,j+2}\cdots \right),
\]
where we choose $\{k_m\}_{m=1}^d\subseteq \bN$ such that $\sum_{m=1}^{d}(j-k_m) = s$. By stacking $h_m$ in parallel, there exists a network $\phi_1:\bR^d \to \bR^{3d}$ with width $d2^{r+1}+3d$ and depth $2\lceil j/r \rceil +1$ such that
\[
\phi_1(\Bx) = (h_1(x_1),\dots,h_d(x_d)), \quad \forall \Bx = (x_1,\dots,x_d) \in Q(j,\delta,d).
\]
Note that the outputs of $\phi_1(\Bx)$ is one-to-one correspondence with $(m_j(\Bx),r_j(\Bx))$ by
\begin{align*}
m_j(\Bx) &= \left( \sum_{\ell=1}^{k_1} 2^{j-\ell}x_{1,\ell}+ \sum_{\ell=k_1+1}^j 2^{j-\ell}x_{1,\ell}, \dots, \sum_{\ell=1}^{k_d} 2^{j-\ell}x_{d,\ell}+\sum_{\ell=k_d+1}^j 2^{j-\ell}x_{d,\ell} \right), \\
r_j(\Bx) &= \left( \Bin 0.x_{1,j+1}x_{1,j+2}\cdots, \dots, \Bin 0.x_{d,j+1}x_{d,j+2}\cdots \right).
\end{align*}
Using this correspondence, by Lemma \ref{interpolate binary}, there exists a network $\phi_{2,i}: \bR^{d+1} \to \bR$ with width at most $4\cdot 2^{(jd-2s)/2} +5 \le 2^{r+2}+5$ and depth $5\cdot2^s+2$ such that $\phi_{2,i}$ interpolate $2^{jd}$ samples:
\[
\phi_{2,i} \left( \left( \sum_{\ell=1}^{k_1} 2^{j-\ell}x_{1,\ell}, \dots, \sum_{\ell=1}^{k_d} 2^{j-\ell}x_{d,\ell} \right), q(\Bx) \right) = b_i (m_j(\Bx)),
\]
where
\[
q(\Bx) = 1+\sum_{\ell=k_1+1}^{j} 2^{j-\ell}x_{1,\ell} + \sum_{m=2}^{d} 2^{\sum_{n=1}^{m-1} (j-k_n)} \sum_{\ell=k_m+1}^{j} 2^{j-\ell}x_{m,\ell} \in \{1,\dots,2^s\}.
\]
Abusing of notation, we denote these facts by 
\begin{align*}
\phi_1(\Bx) &= (\phi_{1,1}(\Bx), \phi_{1,2}(\Bx)), \\
b_i (m_j(\Bx)) &= \phi_{2,i}(\phi_{1,1}(\Bx)), \\
r_j(\Bx) &= \phi_{1,2}(\Bx).
\end{align*}

By assumption, there exists a network $\phi_0$ with width $N_\varphi(\epsilon)$ and depth $L_\varphi(\epsilon)$ such that $\| \varphi - \phi_0\|_{\infty} \le \epsilon\|\varphi\|_\infty$. Thus, $|\phi_0(r_j(\Bx))| \le (1 +\epsilon)\|\varphi\|_\infty \le 2$. Since $b_i (m_j(\Bx)) \in \{0,1\}$, the product
\[
2^{1-i} b_i (m_j(\Bx)) \varphi(r_j(\Bx)) \approx 2^{1-i} \phi_{2,i}(\phi_{1,1}(\Bx)) \phi_0(\phi_{1,2}(\Bx))
\]
can be computed using the observation that, for $a\in \{0,1\}$ and $b\in [-2,2]$,
\begin{equation}\label{product}
4\sigma\left( \frac{b}{4} +a -\frac{1}{2} \right) -2a = \left\{ 
\begin{aligned}
&0\quad a=0\\
&b\quad a=1 
\end{aligned}
\right. \quad =ab,
\end{equation}
which is a network with width $2$ and depth $2$.

Finally, our network function $\phi(\Bx)$ is defined as
\begin{equation}\label{phij}
\phi(\Bx) = \sum_{i=1}^{\lceil \log_2(1/\epsilon)\rceil +1} 2^{1-i} \phi_{2,i}(\phi_{1,1}(\Bx)) \phi_0(\phi_{1,2}(\Bx)) -\phi_0(\phi_{1,2}(\Bx)).
\end{equation}
To implement the summation (\ref{phij}), we can first compute $(\phi_{1,1}(\Bx), \phi_{1,2}(\Bx))$ by the network $\phi_1$, and then compute $(\phi_{1,1}(\Bx), \phi_0(\phi_{1,2}(\Bx)))$ by the network $\phi_0$, then by applying $\tilde{r}$ sub-network $\phi_{2,i}$ and using (\ref{product}), we can compute 
\[
\left(\phi_{1,1}(\Bx), \phi_0(\phi_{1,2}(\Bx)), \sum_{i=1}^{\tilde{r}} 2^{1-i} \phi_{2,i}(\phi_{1,1}(\Bx))\phi_0(\phi_{1,2}(\Bx)) \right).
\]
Since $\lceil \log_2(1/\epsilon)\rceil +1\le \tilde{r}\tilde{s}$, we need at most $\tilde{s}$ such blocks to compute the total summation. The network architecture can be visualized as follows: 
\[
\Bx \longmapsto 
\begin{aligned}
&\phi_{1,1}(\Bx) \\
&\phi_{1,2}(\Bx)
\end{aligned}
\longmapsto 
\begin{aligned}
&\phi_{1,1}(\Bx) \\
&\phi_0(\phi_{1,2}(\Bx))
\end{aligned}
\longmapsto 
\begin{aligned}
&\phi_{1,1}(\Bx) \\
&\phi_0(\phi_{1,2}(\Bx)) \\
&\sum_{i=1}^{\tilde{r}} \Phi_i(\Bx)
\end{aligned}
\longmapsto \cdots \longmapsto 
\begin{aligned}
&\phi_{1,1}(\Bx) \\
&\phi_0(\phi_{1,2}(\Bx)) \\
&\sum_{i=1}^{(\tilde{s}-1)\tilde{r}} \Phi_i(\Bx)
\end{aligned}
\longmapsto \phi(\Bx),
\]
where $\sum_{i=1}^{k} \Phi_i(\Bx)$ represents the summation $\sum_{i=1}^{k} 2^{1-i} \phi_{2,i}(\phi_{1,1}(\Bx))\phi_0(\phi_{1,2}(\Bx))$. According to this construction, in order to compute $\phi$, the required width is at most
\[
\max \{ d2^{r+1} +3d, d+1+N_\varphi(\epsilon), \tilde{r} (2^{r+2}+5)+d+3 \} \le \max\{7d\tilde{r} 2^r,N_\varphi(\epsilon)\}+4d,
\]
and the required depth is at most
\[
2\lceil j/r \rceil + L_\varphi(\epsilon) + \tilde{s} (5\cdot 2^s+2)  \le 4+4\lceil s/d \rceil + 6 \tilde{s} 2^s + L_\varphi(\epsilon)\le 14 \tilde{s} 2^s + L_\varphi(\epsilon).
\]

It remains to estimate the approximation error. For any $\Bx \in Q(j,\delta,d)$, by the definition of $b_i (\Bm)$ (see (\ref{definition of b})), we have
\begin{align*}
\phi(\Bx) = & \sum_{i=1}^{\lceil \log_2(1/\epsilon)\rceil +1} 2^{1-i} \phi_{2,i}(\phi_{1,1}(\Bx)) \phi_0(\phi_{1,2}(\Bx)) -\phi_0(\phi_{1,2}(\Bx)) \\
= & \sum_{i=1}^{\lceil \log_2(1/\epsilon)\rceil +1} 2^{1-i} b_i (m_j(\Bx)) \phi_0(r_j(\Bx)) -\phi_0(r_j(\Bx)) \\
= &\widetilde{c}_{m_j(\Bx)} \phi_0(r_j(\Bx)),
\end{align*}
where $\widetilde{c}_{m_j(\Bx)}/2 +1/2$ is equal to the first $\lceil \log_2(1/\epsilon)\rceil +1$-bits in the binary representation of $c_{m_j(\Bx)} /2+ 1/2 \in [0,1)$. Since $|c_{m_j(\Bx)} -\widetilde{c}_{m_j(\Bx)}|\le \epsilon$ and  $\| \varphi - \phi_0\|_{\infty} \le\epsilon \|\varphi\|_\infty$, we have
\begin{align*}
&|c_{m_j(\Bx)} \varphi(r_j(\Bx)) - \phi(\Bx)| \\
=&\left| c_{m_j(\Bx)} \varphi(r_j(\Bx)) - \widetilde{c}_{m_j(\Bx)} \phi_0(r_j(\Bx)) \right| \\
\le &\left| c_{m_j(\Bx)} \varphi(r_j(\Bx)) - c_{m_j(\Bx)} \phi_0(r_j(\Bx)) \right| +  \epsilon \left| \phi_0(r_j(\Bx)) \right| \\
\le &\epsilon \|\varphi\|_\infty | c_{m_j(\Bx)}| + \epsilon (1+\epsilon) \|\varphi\|_\infty\\
\le &3\epsilon,
\end{align*}
where in the last inequality, we use the assumption $|c_{\Bm}|\le 1$ and $\|\varphi\|_\infty =1$. So we finish the proof.

\section{Proof of Theorem \ref{uniform approximation}} \label{sec:proof_uniform}

Recall that the middle function $\mid(\cdot,\cdot,\cdot)$ is a function that returns the middle value of the three inputs. The following two lemma are from \citet{lu2020deep}.

\begin{lemma}\label{mid error}
For any $\epsilon>0$, if at least two of $\{x_1,x_2,x_3\}$ are in $[y-\epsilon,y+\epsilon]$, then $\mid(x_1,x_2,x_3)\in [y-\epsilon,y+\epsilon]$.
\end{lemma}
\begin{proof}
Without loss of generality, we assume $x_1,x_2\in [y-\epsilon,y+\epsilon]$. If $\mid(x_1,x_2,x_3)$ is $x_1$ or $x_2$, then the assertion is true. If $\mid(x_1,x_2,x_3) = x_3$, then $x_3$ is between $x_1$ and $x_2$, hence $ \mid(x_1,x_2,x_3) = x_3 \in [y-\epsilon,y+\epsilon]$.
\end{proof}

\begin{lemma}\label{mid net}
There exists a ReLU network $\phi$ with width $14$ and depth $3$ such that
\[
\phi(x_1,x_2,x_3) = \mid(x_1,x_2,x_3), \quad x_1,x_2,x_3\in \bR.
\]
\end{lemma}
\begin{proof}
Observe that
\[
\max(x_1,x_2) = \tfrac{1}{2} \sigma(x_1+x_2) - \tfrac{1}{2} \sigma(-x_1-x_2) + \tfrac{1}{2} \sigma(x_1-x_2) + \tfrac{1}{2} \sigma(x_2-x_1).
\]
The function $\max(x_1,x_2,x_3) = \max(\max(x_1,x_2), \sigma(x_3)-\sigma(-x_3))$ can be implemented by a network $\phi_1$ with width $6$ and depth $3$. Similarly, the function $\min(x_1,x_2,x_3)$ can be implemented by a network $\phi_2$ with width $6$ and depth $3$. Therefore,
\[
\mid(x_1,x_2,x_3) = \sigma(x_1+x_2+x_3) - \sigma(-x_1-x_2-x_3) - \max(x_1,x_2,x_3) - \min(x_1,x_2,x_3)
\]
can be implemented by a network with width $14$ and depth $3$.
\end{proof}

Combining these two lemmas with the construction in Proposition \ref{key approximation}, we are now ready to extend the approximation on $Q(j,\delta,d)$ to the uniform approximation on $[0,1]^d$.

\begin{proof}[Proof of Theorem \ref{uniform approximation}]
Without loss of generality, we assume that $M=1$ and $\|\varphi\|_\infty= 1$. To simplify the notation, we let $\{\Be_1,\dots,\Be_d\}$ be the standard basis of $\bR^d$ and denote that $L:= 14 \tilde{s} 2^s + L_\varphi(\epsilon)$ and $N:=\max\{7d\tilde{r} 2^r,N_\varphi(\epsilon)\}+4d$, which are the required depth and width in Proposition \ref{key approximation}, respectively. 

For $k=0,1,\dots,d$, let
\[
E_k := \{ \Bx=(x_1,\dots,x_d)\in[0,1]^d: x_i\in Q(j,\delta,1), i>k \}.
\]
Notice that $E_0 = Q(j,\delta,d)$ and $E_d = [0,1]^d$.

Fixing any $\delta < 2^{-j}/3$, we will inductively construct networks $\Phi_k$, $k=0,1,\dots,d$, with width at most $3^k \cdot 2C_\varphi N$ and depth at most $L+2k$ such that
\[
\|g-\Phi_k\|_{L^\infty (E_k)} \le 6C_\varphi \epsilon.
\]
where $g$ is the target function
\[
g(\Bx) := \sum_{\Bn\in\bZ^d} c_{\Bn} \varphi(2^j\Bx-\Bn)= \sum_{\Bk\in \bZ^d_\varphi} c_{m_j(\Bx)+\Bk} \varphi(2^j\Bx-m_j(\Bx)-\Bk).
\]

For $k=0$, by Proposition \ref{key approximation}, there exists a network $\Phi_0$ with width $C_\varphi N$ and depth $L$ satisfies the requirement.

To construct $\Phi_1$, we observe that for any $\Bx \in Q(j,\delta,d)\pm \delta \Be_1$,
\begin{align*}
g(\Bx) = \sum_{\Bn\in\bZ^d} c_{\Bn} \varphi(2^j\Bx-\Bn) = \sum_{\Bn\in\bZ^d} c_{\Bn} \varphi(2^j\By-\Bn \pm 2^j \delta \Be_1),
\end{align*}
where $\By = \Bx \mp \delta \Be_1 \in Q(j,\delta,d)$. We consider the approximation of the functions
\begin{align*}
g_{\pm \delta \Be_1}(\By) &:= g(\By \pm \delta \Be_1) = g(\Bx) =\sum_{\Bn\in\bZ^d} c_{\Bn} \varphi(2^j\By-\Bn \pm 2^j \delta \Be_1) \\
&= \sum_{\Bm \in \bZ_j^d} \sum_{\Bk\in \bZ^d} c_{\Bm+\Bk} \varphi_{j,\pm \delta \Be_1}(2^j\By-\Bm-\Bk) \cdot 1_{\{\By\in [0,2^{-j})^d+ 2^{-j}\Bm\}} \\
&= \sum_{\Bk\in \bZ^d_{\varphi_{j,\pm \delta \Be_1}}} c_{m_j(\By)+\Bk}  \varphi_{j,\pm \delta \Be_1}(2^j\By-m_j(\By)-\Bk),
\end{align*}
where $\varphi_{j,\pm \delta \Be_1}(\Bx):= \varphi(\Bx\pm 2^j\delta \Be_1)$ and we use the fact that $\varphi_{j,\pm \delta \Be_1}(2^j\By-\Bm-\Bk)$ is nonzero on $[0,2^{-j})^d+ 2^{-j}\Bm$ if and only if $\varphi_{j,\pm \delta \Be_1}(2^j\By-\Bk)$ is nonzero on $[0,2^{-j})^d$ if and only if $\Bk \in \bZ^d_{\varphi_{j,\pm \delta \Be_1}}$. 

For any fixed $j$ and $\Bk\in \bZ^d_{\varphi_{j,\pm \delta \Be_1}}$, replacing $\phi_0(\cdot)$ by $\phi_0(\cdot-\Bk\pm 2^j\delta \Be_1) $ in the construction in section \ref{archetecture}, we can construct a network $\phi^{(j,\Bk)}$ (similar to the representation (\ref{phij})) with width at most $N$ and depth at most $L$ such that it can approximate the function
\[
\By \mapsto c_{m_j(\By)+\Bk}  \varphi_{j,\pm \delta \Be_1}(2^j\By-m_j(\Bx)-\Bk)
\]
with error at most $3\epsilon$ on $Q(j,\delta,d)$.

Observe that $|\bZ^d_{\varphi_{j,\pm \delta \Be_1}}|\le 2C_\varphi$, the function
\[
\Phi_{0,\pm \delta \Be_1 }(\By) := \sum_{\Bk\in \bZ^d_{\varphi_{j,\pm \delta \Be_1}}} \phi^{(j,\Bk)}(\By)
\]
can be computed by $2C_\varphi$ parallel sub-networks with width $N$ and depth $L$. For any $\By\in Q(j,\delta,d)$, the approximation error is
\[
|g_{\pm \delta \Be_1}(\By) - \Phi_{0,\pm \delta \Be_1 }(\By)|\le |\bZ^d_{\varphi_{j,\pm \delta \Be_1}}|\cdot 3\epsilon \le 6C_\varphi \epsilon.
\]

We let
\[
\Phi_1(\Bx) = \mid(\Phi_0(\Bx),\Phi_{0, \delta \Be_1 }(\Bx-\delta \Be_1),\Phi_{0, -\delta \Be_1 }(\Bx+\delta \Be_1)).
\]
By Lemma \ref{mid net} and the construction of $\Phi_0$ and $\Phi_{0,\pm \delta \Be_1}$, the function $\Phi_1$ can be implemented by a network with width $3\cdot 2C_\varphi N$ and depth $L+2$. Notice that for any $\Bx \in E_1$, at least two of $\Bx, \Bx-\delta \Be_1, \Bx+\delta \Be_1$ are in $Q(j,\delta,d)$. Hence, at least two of the inequalities
\begin{align*}
|g(\Bx) - \Phi_0(\Bx)| &\le 6C_\varphi \epsilon,\\
|g(\Bx) - \Phi_{0, \delta \Be_1 }(\Bx-\delta \Be_1)| &= |g_{\delta\Be_1}(\Bx-\delta \Be_1) - \Phi_{0, \delta \Be_1 }(\Bx-\delta \Be_1)|\le 6C_\varphi \epsilon,\\
|g(\Bx) - \Phi_{0, -\delta \Be_1 }(\Bx+\delta \Be_1)| &= |g_{-\delta\Be_1}(\Bx+\delta \Be_1) - \Phi_{0, -\delta \Be_1 }(\Bx+\delta \Be_1)|\le 6C_\varphi \epsilon.
\end{align*}
are satisfied. By Lemma \ref{mid error}, we have
\[
|g(\Bx) - \Phi_1(\Bx)| \le 6C_\varphi \epsilon, \quad  \Bx \in E_1.
\]

Suppose that, for some $k<d$, we have constructed a network $\Phi_k$ with width $3^k \cdot 2C_\varphi N$ and depth $L+2k$. By considering the function
\begin{align*}
g_{\pm \delta \Be_{k+1}}(\By) &:= g(\By \pm \delta \Be_{k+1}) = \sum_{\Bn\in\bZ^d} c_{\Bn} \varphi(2^j\By-\Bn \pm 2^j \delta \Be_{k+1}) \\
&= \sum_{\Bk\in \bZ^d_{\varphi_{j,\pm \delta \Be_{k+1}}}} c_{m_j(\By)+\Bk}  \varphi_{j,\pm \delta \Be_{k+1}}(2^j\By-m_j(\By)-\Bk),
\end{align*}
which has the same structure as $g(\Bx)$ on $E_k$, we can construct networks $\Phi_{k,\pm \delta\Be_{k+1}}$ of the same size as $\Phi_k$ such that
\[
|g_{\pm \delta \Be_{k+1}}(\By) - \Phi_{k,\pm \delta \Be_{k+1} }(\By)| \le 6C_\varphi \epsilon,\quad \By\in E_k.
\]
And by Lemma \ref{mid net}, we can implement the function
\[
\Phi_{k+1}(\Bx) = \mid(\Phi_k(\Bx),\Phi_{k, \delta \Be_{k+1} }(\Bx-\delta \Be_{k+1}),\Phi_{k, -\delta \Be_{k+1} }(\Bx+\delta \Be_{k+1})).
\]
by a network with width $3^{k+1} \cdot 2C_\varphi N$ and depth $L+2k+2$.

Since for any $\Bx \in E_{k+1}$, at least two of $\Bx, \Bx-\delta \Be_{k+1}, \Bx+\delta \Be_{k+1}$ are in $E_k$, by Lemma \ref{mid error}, we have
\[
|g(\Bx) - \Phi_{k+1}(\Bx)| \le 6C_\varphi \epsilon, \quad  \Bx \in E_{k+1}.
\]

In the case $k=d$, the function $\Phi_d$ is a network of depth $L+2d = 14 \tilde{s} 2^s + L_\varphi(\epsilon)+2d$ and width $3^{d} \cdot 2C_\varphi N=3^{d} \cdot 2C_\varphi ( \max\{7d\tilde{r} 2^r,N_\varphi(\epsilon)\}+4d)$. So we finish the proof.
\end{proof}

\section{Proof of Lemma \ref{Bspline app}}

The following lemma, which is from \citet[Lemma 5.3]{lu2020deep}, gives approximation bound for the product function.

\begin{lemma}\label{product app}
For any $N,L\in \bN$, there exists a ReLU network $\Phi_k$ with width $9N+k+7$ and depth $7k(k-1)L$ such that
\[
|\Phi_k(\Bx)-x_1x_2\cdots x_k| \le 9(k-1)(N+1)^{-7kL}, \quad \forall \Bx=(x_1,x_2,\dots,x_k)\in [0,1]^k,\ k\ge 2.
\]
Further more, $\Phi_k(\Bx)=0$ if $x_i=0$ for some $1\le i\le k$.
\end{lemma}

\begin{proof}
We only sketch the network construction, more details can be found in \citet{lu2020deep,yarotsky2017error}. Firstly, we can use the teeth functions to approximate the square function $x^2$, where teeth functions $T_i:[0,1]\to [0,1]$ are defined inductively:
\[
T_1(x)=
\left\{
\begin{aligned}
&2x \quad &&x\le \tfrac{1}{2},\\
&2(1-x)\quad &&x> \tfrac{1}{2},
\end{aligned}
\right.
\]
and $T_{i+1}=T_{i} \circ T_1$ for $i=1,2,\cdots$.  \citet{yarotsky2017error} made the following insightful observation:
\[
\left| x^2 -x +\sum_{i=1}^{s} \tfrac{T_i(x)}{2^{2i}} \right|\le 2^{-2s-2},\quad x\in[0,1].
\]
By choosing suitable $s$, one can construct a network with width $3N$ and depth $L$ to approximate $x^2$ with error $N^{-L}$. Using the fact
\[
xy=2\left( \left( \tfrac{x+y}{2}\right)^2-\left( \tfrac{x}{2}\right)^2 -\left( \tfrac{y}{2}\right)^2 \right),
\]
we can easily construct a new network $\Phi_2(\cdot,\cdot)$ to approximate $(x,y) \mapsto xy$ on $[0,1]^2$. Finally, to approximate the product function $(x_1,x_2,\cdots,x_k)\mapsto x_1x_2\cdots x_k$, we can construct the network $\Phi_k$ inductively: $\Phi_k(x_1,\cdots,x_k):= \Phi_2(\Phi_{k-1}(x_1,\cdots,x_{k-1}),x_k)$.
\end{proof}

If the input domain is $[0,a]^k$ for some $a>0$, we can define $\Phi_{k,a}(\Bx):=a^k\Phi_k(\Bx/a)$, then
\[
|\Phi_{k,a}(\Bx)-x_1x_2\cdots x_k| = a^k \left|\Phi_k\left( \tfrac{\Bx}{a} \right)-\tfrac{x_1}{a}\tfrac{x_2}{a}\cdots \tfrac{x_k}{a} \right|.
\]
Hence, the approximation error is scaled by $a^k$. 
We can approximate the B-spline $\cN_k^d$ using Lemma \ref{product app}.

\begin{proof}[Proof of Lemma \ref{Bspline app}]
We firstly consider the approximation of $\cN_k$. By Lemma \ref{product app}, there exists a network $\widetilde{\phi_1}$ with width $(k+1)(9N+k+6)$ and depth $7(k-1)(k-2)L+1$ such that
\[
\widetilde{\phi_1}(x) = \frac{1}{(k-1)!} \sum_{l=0}^{k}(-1)^l \binom{k}{l} \Phi_{k-1,k+1}(\sigma(x-l),\cdots,\sigma(x-l) ).
\]
And we have the estimate, for $x\in [0,k+1]$,
\begin{align*}
\left|\cN_k(x)-\widetilde{\phi_1}(x)\right| &\le \frac{1}{(k-1)!} \sum_{l=0}^{k} \binom{k}{l} \left|\sigma(x-l)^{k-1} - \Phi_{k-1,k+1}(\sigma(x-l),\cdots,\sigma(x-l) )\right| \\
&\le \frac{2^k}{(k-1)!} (k+1)^{k-1} 9(k-2)(N+1)^{-7(k-1)L} \\
&\le 9 \frac{(2k+2)^k}{(k-1)!} (N+1)^{-7(k-1)L} =:\epsilon.
\end{align*}
Notice that, for $x<0$, $\widetilde{\phi_1}(x)=0=\cN_k(x)$, the estimate is actually true for all $x\in (-\infty,k+1]$. 

To make this approximation global, we observe that $\cN_k(x)\in[0,1]$ with support $[0,k]$. Thus, we can approximate $\cN_k$ by 
\[
\phi_1(x) := \min( \sigma(\widetilde{\phi_1}(x)), \chi(x) ),
\]
where $\chi$ is the indicator function
\[
\chi(x) := \sigma(1-\sigma(-x)) + \sigma(1-\sigma(x - k)) - 1.
\]
Note that $\chi$ is a piece-wise linear function with $\chi(x)=1$ for $x\in [0,k]$ and $\chi(x)=0$ for $x\notin [-1,k+1]$. We conclude that $\phi_1(x)=0$ for $x\notin [0,k+1]$ and
\[
\|\cN_k - \phi_1\|_\infty = \sup_{x\in[0,k+1]}|\cN_k(x) - \phi_1(x)| \le \sup_{x\in[0,k+1]}|\cN_k(x) - \widetilde{\phi_1}(x)| \le \epsilon.
\]
Since the minimum of two number $x,y\in \bR$ can be computed by
\[
\min(x,y) =\tfrac{1}{2} \left(\sigma(x+y) - \sigma(-x-y) +\sigma(x-y) +\sigma(y-x)\right),
\]
$\phi_1$ can be implemented by a network with width $(k+1)(9N+k+6)+2 \le (k+1)(9(N+1)+k)$ and depth $7(k-1)(k-2)L+3\le 7k^2L$.

Recall that
\[
\cN_k^d(\Bx) := \prod_{i=1}^d \cN_k(x_i), \quad \Bx=(x_1,\dots,x_d)\in \bR^d.
\]
Using Lemma \ref{product app}, we can approximate $\cN_k^d$ by 
\[
\phi_d(\Bx) := \Phi_d(\phi_1(x_1),\cdots,\phi_1(x_d)),
\]
which is a network with width $d(k+1)(9(N+1)+k)$ and depth $\le 7(k^2+d^2)L$. Noticing that $\phi_1(x)\in [0,1]$, the approximation error is
\begin{align*}
|\cN_k^d(\Bx) - \phi_d(\Bx)| &\le \left|\prod_{i=1}^d \cN_k(x_i) - \prod_{i=1}^d \phi_1(x_i) \right| + \left|\prod_{i=1}^d \phi_1(x_i) - \Phi_d(\phi_1(x_1),\cdots,\phi_1(x_d)) \right| \\
&\le \left|\prod_{i=1}^d \cN_k(x_i) - \prod_{i=1}^d \phi_1(x_i) \right| + 9(d-1)(N+1)^{-7dL}.
\end{align*}
By repeated applications of the triangle inequality, we have 
\[
\left|\prod_{i=1}^d \cN_k(x_i) - \prod_{i=1}^d \phi_1(x_i) \right| \le \sum_{j=1}^d \left|\prod_{i=1}^{j-1} \phi_1(x_i) \prod_{i=j}^{d} \cN_k(x_i) - \prod_{i=1}^j \phi_1(x_i)\prod_{i=j}^{d} \cN_k(x_i) \right| \le d\epsilon,
\]
where we have use the fact that $\cN_k(x), \phi_1(x)\in [0,1]$ and $\|\cN_k - \phi_1\|_\infty\le \epsilon$.
\end{proof}

\section*{Acknowledgments}

The research of Y. Wang is supported by the HK RGC grant 16308518, the HK Innovation Technology Fund Grant  ITS/044/18FX and the Guangdong-Hong Kong-Macao Joint Laboratory for Data Driven Fluid Dynamics and Engineering Applications (Project 2020B1212030001).

\bibliographystyle{plainnat}
\bibliography{references}

\end{document}